\pdfoutput = 1
\documentclass[11pt]{article}
\usepackage{fullpage}
\usepackage{amsmath}
\usepackage{amsfonts}
\usepackage[ruled]{algorithm2e}
\usepackage{amsthm}
\usepackage{hyperref}
\usepackage{cleveref}
\usepackage[dvipsnames]{xcolor}
\usepackage{natbib}
\usepackage{booktabs}
\usepackage{algorithmic}
\usepackage{url}
\usepackage{graphicx}
\usepackage{subfigure}

\newcommand{\ignore}[1]{}

\newcommand{\cD}{{\mathcal D}}
\newcommand{\cP}{{\mathcal P}}

\newcommand{\cA}{{\mathcal A}}
\newcommand{\cC}{{\mathcal C}}
\newcommand{\bE}{{\mathbb E}}
\newcommand{\bR}{{\mathbb R}}
\newcommand{\bZ}{{\mathbb Z}}

\newcommand{\one}{{\mathbf{1}}}

\theoremstyle{plain}
\newtheorem{theorem}{Theorem}[section]

\newtheorem{lemma}[theorem]{Lemma}
\newtheorem{corollary}[theorem]{Corollary}

\theoremstyle{definition}
\newtheorem{definition}[theorem]{Definition}

\theoremstyle{remark}
\newtheorem{remark}[theorem]{Remark}

\newcommand{\mae}{{\operatorname{\mathsf{MA-err}}}}

\usepackage{caption}
\usepackage{subcaption} 

\usepackage[dvipsnames]{xcolor}
\usepackage{microtype}
\usepackage{graphicx}
\usepackage{booktabs} %

\title{Multigroup Robustness}
\date{}
\author{Lunjia Hu\thanks{Stanford University. Supported by Moses Charikar’s and Omer Reingold's Simons Investigators awards and the Simons Foundation Collaboration on the Theory of Algorithmic Fairness. Email: \texttt{lunjia@stanford.edu}}\qquad Charlotte Peale \thanks{Stanford University. Supported by the Simons Foundation Collaboration on the Theory of Algorithmic Fairness. Email: \texttt{cpeale@stanford.edu}} \qquad Judy Hanwen Shen\thanks{Stanford University. Supported by the Simons Foundation Collaboration on the Theory of Algorithmic Fairness. Email: \texttt{jhshen@stanford.edu}}}

\allowdisplaybreaks
\begin{document}

\maketitle
\begin{abstract}
    To address the shortcomings of real-world datasets, robust learning algorithms have been designed to overcome arbitrary and indiscriminate data corruption. However, practical processes of gathering data may lead to patterns of data corruption that are localized to specific partitions of the training dataset. Motivated by critical applications where the learned model is deployed to make predictions about people from a rich collection of overlapping subpopulations, we initiate the study of \emph{multigroup robust} algorithms whose robustness guarantees for each subpopulation only degrade with the amount of data corruption \emph{inside} that subpopulation. When the data corruption is not distributed uniformly over subpopulations, our algorithms provide more meaningful robustness guarantees than standard guarantees that are oblivious to how the data corruption and the affected subpopulations are related. Our techniques establish a new connection between multigroup fairness and robustness.
\end{abstract}

\section{Introduction}
The gap between standard distributional assumptions and practical dataset limitations has been well-studied in machine learning literature -- from unintended distribution shift to adversarially crafted data manipulations. 
Corresponding notions of \emph{robustness} have been developed to reflect the goal of learning well in the presence of adversarially corrupted data, as well as attacks demonstrating that small amounts of corrupted data can seriously impact performance.

While attacks that target specific subgroups have been proposed~\citep{jagielski2021subpopulation}, they give the adversary the power to change any point in the dataset to achieve its goal--even the ability to modify seemingly unrelated points outside the subgroup of interest. While this is reasonable as a strong worst-case adversarial assumption, data corruption issues are often far more localized in reality. In surveys, response bias may compromise answers from certain subpopulations \citep{meng2018statistical, bradley2021unrepresentative}. As another example, when amassing internet data for training large models, certain sources can be less trustworthy or more toxic than others~\citep{dodge2021documenting}. Since limitations to datasets may be isolated to certain groups, it is important to develop a fine-grained notion of robustness that ensures groups do not suffer undue harm due to corruption in unrelated data.
\paragraph{A fine-grained robustness notion} We consider a new data-aware notion of robustness that we term \emph{multigroup robustness}. At a high-level, a multigroup-robust learning algorithm guarantees that the effects of dataset corruption on every subpopulation-of-interest are bounded by the amount of corruption to data within that subpopulation (Figure~\ref{fig:multigroup-robustness-diagram}).

More formally, we consider a binary-label learning problem where data from a domain $X$ labeled with $y \in \{0, 1\}$ is inputted into a deterministic learning algorithm $\cA: (X \times \{0, 1\})^* \rightarrow [0, 1]^X$ that outputs a predictor $p \in [0, 1]^X$. Given a set of subpopulations $\cC \subseteq 2^X$, we say that $\cA$ is multigroup robust with respect to $\cC$ if, given a dataset $S = \{(x_1, y_1), ..., (x_n, y_n)\}$ where the $x_i$s are each drawn i.i.d. from some unknown distribution $\cD_X$, given any other potentially corrupted dataset $S' = \{(x_1', y_1'), ..., (x_m', y_m')\}$, we can guarantee that the difference in the mean prediction of \emph{every} group $C \in \cC$ is bounded by the amount of change to $C$ between $S$ and $S'$, i.e. 
\[\left|\bE_{\cD_X}[\left(\cA(S)(x) - \cA(S')(x)\right)\one[x \in C]]\right| \leq \mathsf{dist}_C(S, S')\]
where $\mathsf{dist}_C(S, S')$ is a measure of the changes to the set of points belonging to $C$ between $S$ and $S'$, which we will formalize in Section~\ref{sec:multi-robust}. 

Intuitively, this definition asks that when the points from a group are unchanged or change very little in $S'$ (i.e. translating to a small $\mathsf{dist}_C(S, S')$), then the average prediction outputted by $\cA$ on that group should also not change by much, thus preserving the group's accuracy-in-expectation. %
\begin{figure}
    \centering
    \includegraphics[width=0.5\textwidth]{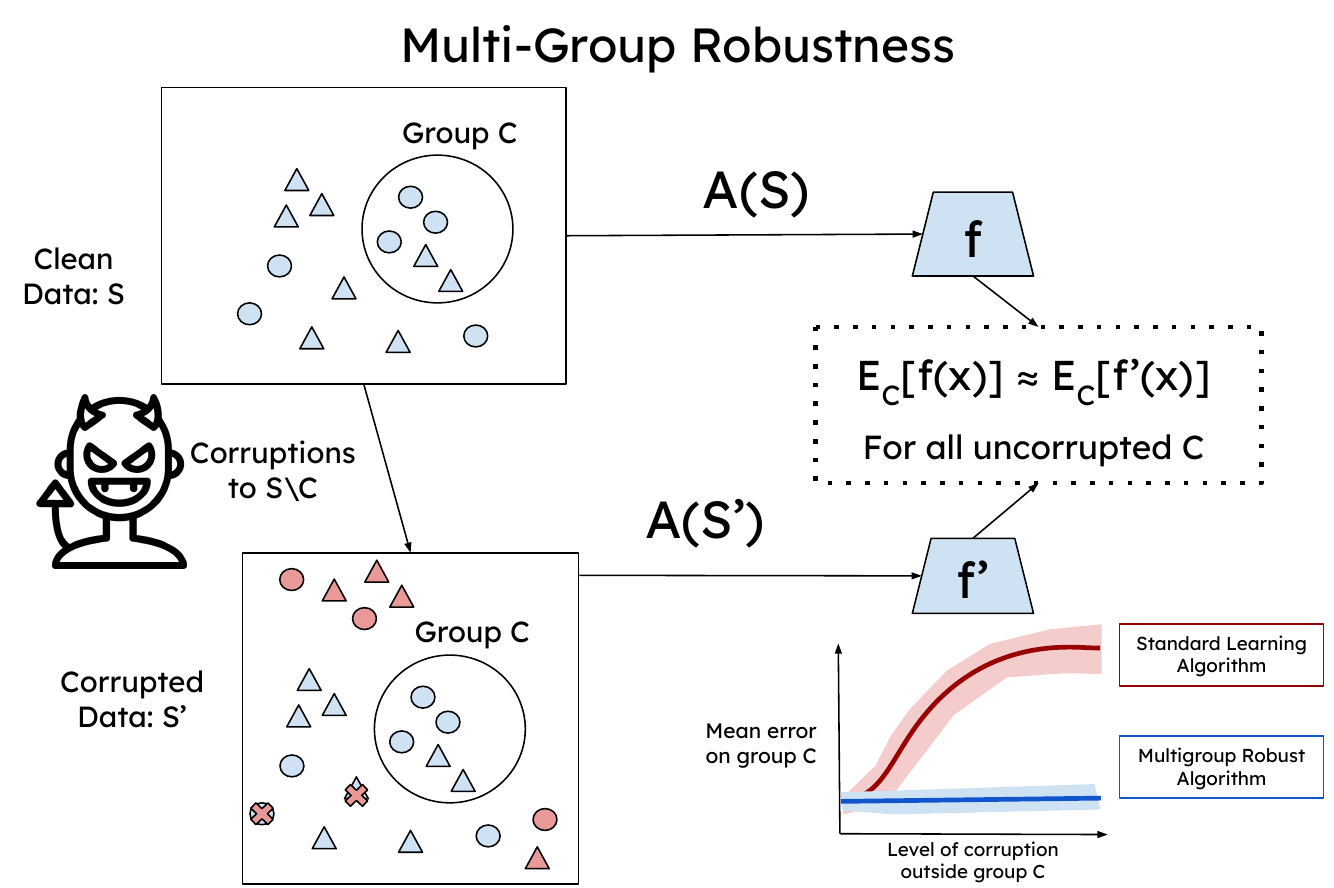}
    \caption{\small Intuitive illustration of multigroup robustness: for every group $C$, if points within the group are not modified, a multigroup robust algorithm produces a predictor that achieves marginal mean consistency with the clean data predictor (See Definition \ref{def:multigroup-robust}).}
    \label{fig:multigroup-robustness-diagram}
\end{figure}
It is easy to construct algorithms that satisfy multigroup robustness, but are not very useful as learning algorithms; for example, the algorithm that always outputs the same predictor for any dataset. Thus, we will concentrate on the particular question of whether there exist efficient learning algorithms that provide multigroup robustness as well as an agnostic learning guarantee, i.e., that the predictor outputted by $\cA$ performs at least as well as the best predictor from some benchmark set $\cP \subseteq [0, 1]^X$. 

\paragraph{Standard algorithms are not multigroup robust} Classic agnostic learning algorithms such as empirical risk minimization over the benchmark class can fail to satisfy multigroup robustness even for extremely simple families of subpopulations. Consider as an example the benchmark class containing the all ones or all zeros predictors, $\cP = \{p_0, p_1\}$, and a uniform distribution over $X$ where half of the points ($x \in X_0$) are labeled with 0s, and half of the points ($x \in X_1$) are labelled with ones. Given a sample from this half-and-half distribution, ERM will output the predictor matching the majority label in the dataset. However, because most datasets drawn will be close to half ones and half zeros, this means that only a tiny number of corrupted labels can drastically change the output from $p_0$ to $p_1$ or vice versa, impacting the predictions of a huge number of points despite not changing the overall accuracy by much. 

In Section \ref{sec:experiments}, we empirically demonstrate that several standard models for classification fail to preserve multigroup robustness under simple label-flipping and data addition attacks on the Adult Income Dataset. 

\paragraph{Connections to multiaccuracy} In light of these successful attacks, new ideas are necessary to develop algorithms that can match the performance guarantees of standard learning approaches while being multigroup robust.

When we only care about a small number of disjoint groups, a naive solution that provides multigroup robustness could be to simply train a separate model on each group's data. However, this approach becomes inefficient as we consider a growing number of possibly overlapping groups. Moreover, the individually trained models may lose out on the predictive power that could be gained from using the dataset as a whole. 

Instead, we turn to the notion of \emph{multiaccuracy}, a learning objective originating in the algorithmic fairness literature that asks for a predictor to satisfy accuracy-in-expectation simultaneously on many groups~\citep{hebert2018multicalibration, kim2019multiaccuracy}. In Section~\ref{sec:ma-and-robust} we outline how multiaccuracy's rigorous group-level guarantees can connect to our goal of multigroup robustness. However, because standard multiaccuracy algorithms assume access to i.i.d. data, we encounter challenges to using them in our setting where due to adversarial corruption, we cannot assume data is i.i.d.. We demonstrate how to bypass these obstacles by making appropriate modifications to standard multiaccuracy algorithms, and present a set of sufficient conditions for algorithms that provably achieve multigroup robustness. In Section~\ref{sec:lower-bd}, we supplement our results with a lower bound showing that multiaccuracy is a necessary property of any non-trivial algorithm that satisfies multigroup robustness.

\paragraph{Achieving multigroup robustness} With these sufficient conditions in hand, in Section~\ref{sec:implementing} we present an efficient post-processing approach that can be used to augment any existing learning algorithm to add both multigroup robustness and multiaccuracy guarantees, while preserving the performance guarantees of the original learning algorithm.

In Section~\ref{sec:experiments}, we supplement our theoretical results with experiments on real-world census datasets demonstrating that our post-processing approach can be added to existing learning algorithms to provide multigroup robustness protections without a drop in accuracy. 

\subsection{Our Contributions}
To summarize, we list our main contributions below:
\begin{itemize}
    \item Define a new notion of data-aware robustness, \emph{multigroup robustness}, that ensures subgroups do not suffer undue harm due to corruptions in unrelated data (Section~\ref{sec:multi-robust}).  
    \item Demonstrate general sufficient conditions for an algorithm to be multigroup robust by drawing connections to multiaccuracy (Section~\ref{sec:ma-and-robust}), and show that multiaccuracy is necessary for non-trivial multigroup robust algorithms (Section~\ref{sec:lower-bd}). 
    \item Present an efficient post-processing algorithm that can provide an arbitrary learning algorithm with multigroup robustness guarantees while preserving performance (Section~\ref{sec:implementing}). 
    \item Empirically validate that standard learning algorithms are vulnerable to simple attacks on multigroup robustness, and that our postprocessing method successfully protects against these attacks while preserving accuracy (Section~\ref{sec:experiments}). 
\end{itemize}

\section{Related Work}

\paragraph{Fairness and Data Poisoning}
Prior works have demonstrated that the fairness properties of machine learning models can be degraded by modifying small subsets of the training data~\citep{solans2020poisoning, van2022poisoning, chai2023robust}. \citet{jagielski2021subpopulation} show that subpopulations can be directly targeted in data poisoning attacks. Algorithms for finding fair predictors that are robust to data poisoning to any part of the dataset have also been proposed for regression~\citep{jin2023fairness}. ``Subgroup robustness" in recent literature refers to the performance of the worst~\citep{martinez2021blind, gardner2022subgroup}. In distribution shift literature, ``subgroup robustness" is used interchangeably with ``worst group robustness" \citep{sagawa2019distributionally}. In contrast, our definitions focus on how modifications to some groups in the training data can impact other unrelated groups. 

\paragraph{Multiaccuracy and Multicalibration}

Multiaccuracy and multicalibration are multigroup fairness notions that require a predictor to provide meaningful statistical guarantees (e.g., accuracy in expectation, calibration) on a large family of possibly overlapping subgroups of a population ~\citep{tradeoff, hebert2018multicalibration, gerrymander, kim2019multiaccuracy}. 
\citet{uniadapt} show that multicalibration can ensure a predictor's robustness against distribution shift, achieving \emph{universal adaptability}. They focus on covariate shift while assuming the conditional distribution of $y$ given $x$ remains the same and assume that the covariate shift can be represented by a propensity score function from a given family. Our work considers general forms of data corruption, both in covariate $x$ and label $y$, and the corrupted data need not be i.i.d.\ from any distribution.

\paragraph{Robust Statistics}
Algorithms for estimating a variety of statistical quantities on corrupted data have been studied extensively in the literature, where the quantities to estimate include mean \citep{robust-mean-1,robust-mean-2}, covariance matrix \citep{robust-mean-1,robust-gaussian}, principal components \citep{robust-pca},  and beyond.
In the typical setup of robust statistics, the corrupted data is formed by modifying or hiding a significant fraction of an otherwise i.i.d.\ dataset. The goal is to ensure that the estimation error is small relative to the fraction of corrupted data.

\section{Preliminaries and Notation}

A \emph{deterministic learning algorithm} is defined as any algorithm $\cA: (X \times \{0, 1\})^* \rightarrow \cP$ that takes a sample of data points $(x_1,y_1), ..., (x_n, y_n) \in X \times \{0, 1\}$ for any $n \in \bZ^+$ as input and outputs a predictor $p \in \cP$ where $X$ is a finite domain of points. 

Given two datasets $S = \{(x_1, y_1), ..., (x_n, y_n)\}$ and $S' = \{(x_1', y_1'), ..., (x_m', y_m')\}$, we use the notation $S \Delta_X S'$ to denote a set containing the \emph{symmetric difference} in terms of the multisets $\{x_1, ..., x_n\}$ and $\{x_1', ..., x_m'\}$. 

More formally, because these datasets could have duplicates of certain $x$-values, we cannot use a standard set definition of symmetric difference. This can be thought of as reinterpretting each $\{x_1, ..., x_n\}$ as a map $\mu$ from $X$ to $\mathbb{Z}_{\geq 0}$ where $\mu(x)$ is equal to the number of times $x$ appears in the sample, $\mu(x) = |\{i\in \{1, \ldots , n\}: x_i = x\}|$. Having similarly defined such a $\mu'$ for $\{x_1', ..., x_m'\}$, the multiset symmetric difference of these two collections is defined as the map $\mu_{\Delta}$ such that for all $x \in X$, $\mu_{\Delta}(x) = |\mu(x) - \mu'(x)|$.

Given two distributions $\cD, \cD'$ over a discrete domain $X$ and a subpopulation $C \subseteq X$, we denote the statistical distance between $\cD$ and $\cD'$ restricted to $C$ as 
\begin{align*}
    &\Delta_C(\cD, \cD') := \\
    &\sum_{x \in C }\left|\Pr_{X \sim \cD}[X = x] - \Pr_{X \sim \cD'}[X = x]\right|.
\end{align*}
\section{Multigroup Robustness}
\label{sec:multi-robust}
In this section, we define our notion of \emph{multigroup robustness}. We consider a setting where a learner is provided with a sample of binary-labelled data points $(x_1, y_1), ..., (x_n, y_n) \in \mathcal{X} \times \{0, 1\}$. %
We assume that the points could have been adversarially corrupted. 

We consider two different levels of adversarial power. The strongest adversary we consider can make label-change adjustments to this dataset by replacing a point $(x, y)$ with a new $(x, y')$, as well as add or delete new points to the dataset. Our main definition is designed to protect against these strong data-dependent adversaries, and we state it formally below:

\begin{definition}[Binary-label Multigroup Robustness]
    Let $\cC$ be a subpopulation class consisting of subsets $C\subseteq X$. For any $n \in \bZ^{+}, \varepsilon > 0, \delta \in [0, 1]$,  we say that a deterministic learning algorithm $\mathcal{A}: (X \times \{0, 1\})^* \rightarrow \cP$ is \emph{$(\cC, n, \varepsilon, \delta)$-multigroup robust} if for every distribution $\mathcal{D}_X$ over $X$, the following holds with probability at least $1 - \delta$ over $X_n = (x_1, ..., x_n)$ drawn i.i.d. from $\mathcal{D}_X$: for any $(y_1, ..., y_n) \in \{0, 1\}^n$ and $(x_1', y_1'), ..., (x_m', y_m') \in (X \times \{0, 1\})^m$, let $S$ and $S'$ denote the two samples $\{(x_i, y_i)\}_{i = 1}^n$ and $\{(x_i', y_i')\}_{i = 1}^m$, respectively.
   
    Defining $p := A(S)$ and $p' := A(S')$, we have 
    \begin{equation}
    \label{eq:robust-label}
    \begin{aligned}
    &\left|\mathbb{E}_{x\sim \mathcal{D}_X}[(p(x) - p'(x))\one(x\in C)]\right| \\
    &\le \frac 1 n \left|\sum_{i = 1}^n y_i \one[x_i \in C] - \sum_{j = 1}^m y_j' \one[x_j \in C]\right| \\
    &\,+ \frac 1 n \left|(S \Delta_X S') \cap C\right|+ \varepsilon
    \end{aligned}
    \end{equation}
    for every $C\in \mathcal{C}$.
    \label{def:multigroup-robust}
\end{definition}

We highlight that our definition makes \emph{no distributional assumptions} about the ground-truth $y$-values in the original dataset $S$, meaning that multigroup robustness implies a strong distribution-free robustness property that holds even when the original $y$-values were not i.i.d.. 

Here, the abstract $\mathsf{dist}_C(S, S')$ used in the introduction is formalized to capture label flipping performed by the adversary (the first term) as well as addition and deletion of points (the second symmetric difference term). This means that this definition can also give a definition of robustness in settings where the adversary can only flip labels.

We also consider a weaker adversary that can only change the \emph{distribution} the data is drawn from, and the particular training set is still drawn i.i.d. from that distribution:

\begin{definition}[Binary-label Multigroup Robustness to Distribution Shift]
    Let $\cC$ be a subpopulation class consisting of subsets $C\subseteq X$. For any $n \in \bZ^+$, $\varepsilon > 0$, $\delta \in [0, 1]$, we say that a deterministic learning algorithm $\mathcal{A}: (X \times \{0, 1\})^* \rightarrow \cP$ is \emph{$(\cC, n, \varepsilon, \delta)$-multigroup robust to distribution shift} if for any two distributions $\cD, \cD'$ over $X \times \{0, 1\}$, the following holds with probability at least $1 - \delta$ over $S = \left((x_1, y_1), ..., (x_n, y_n)\right)$, $S' = \left((x'_1, y_1'), ..., (x_n', y_n') \right)$ drawn i.i.d. from $\cD$ and $\cD'$:
   
    Defining $p := A(S)$ and $p' := A(S')$, we have 
    \begin{equation}
    \label{eq:robust-dist-shift}
    \begin{aligned}
    &\left|\mathbb{E}_{x\sim \mathcal{D}_X}[(p(x) - p'(x))\one(x\in C)]\right| \\
    &\le \left|\bE_{(x, y) \sim \cD}[y\one[x \in C]]- \bE_{(x', y') \sim \cD'}[y'\one[x' \in C]]\right|\\
    &+\Delta_C(\cD_X, \cD'_X) + \varepsilon
    \end{aligned}
    \end{equation}
    for every $C\in \mathcal{C}$.
\end{definition}

Note that here, the $\mathsf{dist}_C(S, S')$ bound has been replaced with a measure of the distance between the corrupted and uncorrupted distribution, however the first term still captures the amount of label shift while the second term speaks to the covariate shift in the corrupted distribution.

\section{Multigroup Robustness from Multiaccuracy}\label{sec:ma-and-robust}
In this section, we give principled conditions for algorithms that provably achieve multigroup robustness. Our proofs are largely inspired by previous algorithms for \emph{multiaccuracy}. We identify several key challenges in using those algorithms in our setting with corrupted data, and address these challenges by making appropriate modifications.

We first introduce the definition of multiaccuracy~\citep{hebert2018multicalibration}, a notion of algorithmic fairness that requires a predictor be accurate in expectation for every group in some collection of subgroups $C \subseteq 2^X$. Formally, we define multiaccuracy as follows:

\begin{definition}[Multiaccuracy \citep{hebert2018multicalibration}]
    Let $\cC$ be a subpopulation class consisting of subsets $C\subseteq X$ and $\cD$ some target distribution over $X \times \{0, 1\}$. For $\varepsilon > 0$, we say a predictor $p: X \rightarrow [0, 1]$ is $(\cC, \varepsilon)$-multiaccurate (MA) on $\cD$ if for every $C \in \cC$: 
    \begin{equation}
    \label{eq:ma}
        \left|\bE_{(x, y) \sim D}\left[(y - p(x))\one[x \in C]\right]\right| \leq \varepsilon.
    \end{equation}
\end{definition}

Throughout the paper, it will be useful to talk about algorithms that provide multiaccuracy guarantees. We formally define such an algorithm as follows:
\begin{definition}[Multiaccurate Learning Algorithm]
    Let $\cC \subseteq 2^X$ be a subpopulation class. Given $n \in  \bZ^+, \varepsilon > 0$, $\delta \in [0, 1]$, we say that a deterministic learning algorithm $\cA: (X \times \{0, 1\})^* \rightarrow [0,1]^X$ is $(\cC, n, \varepsilon, \delta)$-multiaccurate if for any distribution $\cD$ over $X \times \{0, 1\}$, given $n$ i.i.d. data points $S = (x_1, y_1), ..., (x_n, y_n)$ from $\cD$, the predictor $\cA(S)$ satisfies $(\cC, \varepsilon)$-MA on $\cD$ with probability at least $1 - \delta$. 
\end{definition}

\subsection{Standard Multiaccuracy Gives Multigroup Robustness under Distribution Shift}

We begin by considering a setting where an adversary can only corrupt the data \emph{distribution} rather than the sampled datapoints themselves, and show that in this setting, learners that output multiaccurate predictors with respect to a collection $\cC$ are also multigroup robust to distribution shift with respect to $\cC$. 

For ease of notation, for a given distribution $\cD$, predictor $p$, and subpopulation $C \subseteq X$ we denote $\mae_{\cD}(p, C) := \bE_{(x, y) \sim \cD}[(y - p(x))\one[x \in C]]$. 

\begin{lemma}[Robustness from MA]\label{lem:ma-rob}
Given two distributions $\cD, \cD'$ over $X \times \{0, 1\}$ and a collection of subsets $\cC \subseteq 2^X$, let $p$ and $p'$ be $(\cC, \varepsilon)$-MA predictors with respect to $\cD$ and $\cD'$, respectively. Then it holds that for all $C \in \cC$,
\begin{equation}\label{eq:ma-lemma}
    \begin{aligned}
        &\left|\bE_{\cD}[(p(x) - p'(x))\one[x \in C]]\right| \\&\leq
        \left| \mae_{\cD}(p', C) - \mae_{\cD'}(p', C)\right|+ 2\varepsilon.
    \end{aligned}
\end{equation}
\end{lemma}

Intuitively, the left-hand-side of equation~\ref{eq:ma-lemma} already bears resemblance to a multigroup robustness statement with respect to $p$ and $p'$. By reinterpretting the upper bound, we can show that multiaccurate learners provide multigroup robustness to distribution shift: 

\begin{lemma}\label{lem:ma-robust-dist-shift}
    Let $\cC \subseteq 2^X$ be a collection of subpopulations and let $\cA:(X \times \{0, 1\})^n \rightarrow [0,1]^X$ be a deterministic learning algorithm satisfying $(\cC, n, \epsilon, \delta)$-MA. Then, $\cA$ is also $(\cC, n, 2\varepsilon, 2\delta)$-multigroup robust to distribution shift. 
\end{lemma}

\subsection{Leveraging Uniform Convergence for Stronger Robustness}

In total, we have shown so far that learning algorithms for multiaccuracy such as those of~\citep{hebert2018multicalibration, kim2019multiaccuracy} \emph{additionally} give out-of-the-box multigroup robustness guarantees against weak adversaries that can corrupt the data distribution. 
This result relies on the key assumption that while the corrupted distribution may be arbitrarily warped compared to the original distribution, we can still assume that the training data is drawn i.i.d. from the corrupted distribution. This is a key property that allows us to apply multiaccuracy algorithms that assume access to an i.i.d. datasource. 

Moving beyond distributional shifts, we are also interested in stronger adversaries that can directly corrupt a data sample either through label change or addition/deletion of points. In the presence of these stronger adversaries, we can no longer assume access to an ``clean'' i.i.d. datasource. 
In the absence of i.i.d. data, we can still work with the empirical distribution: the uniform distribution over the data points. However, because multigroup robustness is a statement about the predictor similarity over the \emph{true marginal distribution $D_X$}, it's not clear whether relying on the corrupted empirical distribution alone can give us these distributional guarantees necessary for robustness. 

Despite this obstacle, our main result shows that by leveraging an additional uniform-convergence assumption, we can in fact achieve multigroup robustness against strong dataset-dependent adversaries for any algorithm that guarantees multiaccuracy on the empirical distribution. 
Our analysis for handling adversarially corrupted data differs from typical analyses where the uniform convergence assumption is applied to a learning algorithm with i.i.d.\ input data (see \Cref{rem:uniform-cvg}).

We begin by showing in Lemma~\ref{lem:ma-ptwise-rob} that guaranteeing empirical multiaccuracy already yields a guarantee of predictor-closeness when measured over the empirical distribution of the uncorrupted dataset. We formally define an empirically robust learning algorithm as follows:

\begin{definition}[Empirically Multiaccurate Learning Algorithm]\label{def:emp-ma}
    Let $\cC$ be a subpopulation class consisting of subsets $C \subseteq X$. Given $\varepsilon > 0$, we say that a deterministic learning algorithm $\cA: (X \times \{0, 1\})^* \rightarrow [0,1]^X$ is $(\cC, \varepsilon)$-empirically-multiaccurate if when given as input data points $S = (x_1, y_1), ..., (x_n, y_n)$ from $X \times \{0, 1\}$ for any $n \in \bZ^+$, the predictor $p:= \cA(S)$ satisfies 
    \[\left|\frac{1}{n}\sum_{i = 1}^n (y_i - p(x_i))\one[x_i \in C]\right| \leq \varepsilon\]
    for all $C \in \cC$. 
\end{definition}

\begin{lemma}[Pointwise Robustness from Empirical MA]\label{lem:ma-ptwise-rob}
    Given two datasets $S = \{(x_i, y_i)\}_{i = 1}^n$, $S' = \{(x_j', y_j')\}_{j = 1}^m$ from $X \times \{0, 1\}$ and a collection of subsets $\cC$, let $p$ and $p'$ be $(\cC, \varepsilon)$-MA and $(\cC, \varepsilon')$-MA predictors with respect to $\mathsf{Uni}(S)$ and $\mathsf{Uni}(S')$, respectively. Then it holds that for all $C \in \cC$, 
    \begin{equation}\label{eq:pointwise-rob-lem}
    \begin{aligned}
        &\left|\bE_{(x, y) \sim \mathsf{Uni}(S)}[(p'(x) - p(x))\one[x \in C]]\right| \\
        &\leq
        \frac{1}{n}\left|\sum_{(x, y) \in S}y \one[x \in C] - \sum_{(x', y') \in S'}y'\one[x' \in C]\right|\\
        &\quad + \frac{1}{n}|(S \Delta_X S')\cap C| + \varepsilon + \frac{m}{n}\varepsilon'.
    \end{aligned}
    \end{equation}
\end{lemma}

While the right side of the statement of Lemma~\ref{eq:pointwise-rob-lem} matches the definition of multigroup robustness, the left side is still an expectation over the empirical distribution $\mathsf{Uni}(S)$ rather than the uncorrupted target distribution $\cD_X$. This means that to show that multigroup robustness holds, it suffices to show that the empirical quantity $\left|\bE_{(x, y) \sim \mathsf{Uni}(S)}[(p'(x) - p(x))\one[x \in C]]\right|$ is close to its population limit, $\left|\bE_{x \sim D_x}[(p'(x) - p(x))\one[x \in C]]\right|$, for all $C \in \cC$. 

Our main result uses this reasoning to show that a learning algorithm that outputs a predictor multiaccurate with respect to the empirical distribution satisfies multigroup robustness whenever we can guarantee \emph{uniform convergence} over all possible outputted predictors and subpopulations (See Definition~\ref{def:uniform-cvg} for a formal definition). 
We additionally show that under this assumption we can guarantee the outputted predictor is multiaccurate with respect to the target distribution when the learning algorithm is given uncorrupted i.i.d. data (Theorem~\ref{thm:general-robustness-and-ma}). 

\begin{definition}\label{def:uniform-cvg}
    Let $\cC \subseteq 2^X$ be a class of subpopulations, and let $\cP \subseteq [0, 1]^X$ be a family of predictors. We say that $\cP$ satisfies $(\cC, n, \varepsilon, \delta)$-\emph{uniform convergence} if for any distribution $\cD$ over $X \times \{0, 1\}$, we are guaranteed that with probability at least $1 - \delta$ over the randomness of datapoints $(x_1, y_1), ..., (x_n, y_n)$ drawn i.i.d. from $\cD$, we are guaranteed that the following inequalities hold for all $p \in \cP$ and $C \in \cC$:
    \begin{align}
    \left|\frac 1n \sum_{i=1}^n p(x_i)\one(x_i\in C) - \bE_\cD [p(x)\one(x\in C)]\right| & \le \varepsilon,\\
    \left|\frac 1n \sum_{i=1}^n y_i\one(x_i\in C) - \bE_\cD [y\one(x\in C)]\right| & \le \varepsilon.
    \end{align}
\end{definition}

\begin{theorem}\label{thm:general-robustness-and-ma}
Let $\cA: (X \times \{0, 1\})^* \rightarrow \cP$ be a deterministic learning algorithm that outputs predictors from the family $\cP \subseteq [0,1]^X$. For a family of subpopulations $\cC \subseteq 2^X$, suppose that $\cA$ satisfies empirical $(\cC, \varepsilon_1)$-multiaccuracy
and additionally suppose that $\cP$ satisfies $(\cC, n, \varepsilon_2, \delta_2)$-uniform convergence. Then, $\cA$ is $(\cC, n, \left(1 + \frac{m}{n}\right)\varepsilon_1 + 2\varepsilon_2, \delta_2)$-multigroup robust as well as $(\cC, n, \varepsilon_1 + 2\varepsilon_2, \delta_2)$-multiaccurate.
\end{theorem}

\begin{remark}[Strong use of uniform convergence]\label{rem:uniform-cvg}
Uniform convergence is a standard technique for establishing the generalization of a learning algorithm and is typically applied to algorithms given uncorrupted, i.i.d. training data. The important difference in our approach is that we apply uniform convergence to algorithms whose input data is adversarially corrupted and cannot be treated as i.i.d.\ from any distribution. In our setting, i.i.d.\ data is first given to an \emph{unrestricted} adversary to produce corrupted data, which is then given to the learning algorithm whose output model is \emph{restricted} to a class. Since we do not restrict the behavior of the adversary, our analysis makes a stronger use of uniform convergence  than typical analyses. 

\end{remark}
\subsection{Lower Bounds}\label{sec:lower-bd}
We now explore whether multiaccuracy is a necessary condition of a multigroup robust algorithm under a weak non-triviality assumption (see details in Appendix~\ref{sec:lower-bd-app}).

\begin{theorem}[Lower Bound]\label{thm:lower-bound}
Let $\cC$ be a class of subpopulations, $\cP \subseteq [0, 1]^X$ a family of predictors containing the all ones predictor $p(x) = 1$ for all $x \in X$, and $\cA: (X \times \{0, 1\}^*) \rightarrow \cP$ a deterministic learning algorithm. If $\cA$ is $(\cC, n, \varepsilon_1, \delta_1)$-multigroup robust and $(n, \varepsilon_2, \delta_2)$-accurate-in-expectation, and $\cP$ satisfies $(\cC, n, \epsilon_3, \delta_3)$-uniform convergence, then $\cA$ is a $(\cC, n,\epsilon_1 + \epsilon_2 + 2 \epsilon_3, 2\delta_1 + 2\delta_2 + \delta_3)$-multiaccurate learning algorithm. 
\end{theorem}

\section{Implementing Multigroup Robustness}\label{sec:implementing}

So far, we have demonstrated sufficient conditions for a learning algorithm to satisfy multigroup robustness (empirical multiaccuracy and uniform convergence). We now show that multigroup robustness can be achieved efficiently in parallel with standard accuracy objectives. In particular, we present a post-processing procedure that can convert any black-box learning algorithm into a multigroup robust learning algorithm that minimizes $\ell_2$ error.

\subsection{Post-processing Approach}

We consider a setting where we are given access to an arbitrary deterministic learning algorithm $\cA: (X \times \{0, 1\})^* \rightarrow [0, 1]^X$. We will demonstrate how to post-process this algorithm to produce a new algorithm $\mathsf{PP}_{\cA}: (X \times\{0, 1\})^* \rightarrow [0, 1]^X$ (Algorithm~\ref{alg:ma-empirical}) that provides comparable performance to that of $\cA$ in terms of $\ell_2$ error while also satisfying multigroup robustness. We present our post-processing approach in Algorithm~\ref{alg:ma-empirical}.

\begin{algorithm}
    \caption{Multiaccuracy Boost on Empirical Distribution}
    \label{alg:ma-empirical}
    \begin{algorithmic}
       \STATE {\bfseries Parameters: }$n \in \mathbb{Z}_{\ge 0}$, $\varepsilon \in \mathbb{R}_{\geq 0}$, $\mathcal{C} \subseteq \{0, 1\}^X$
       \STATE {\bfseries Input: }data points $(x_1, y_1), ..., (x_n, y_n) \in X \times \{0, 1\}$, learning algorithm $\cA:\{0, 1\}^n \rightarrow [0, 1]^X$
       \STATE {\bfseries Output: }predictor $p:X\rightarrow[0,1]$

       \STATE {\bfseries Step 1: Training}
       \STATE  initialize $p \leftarrow \cA((x_1, y_1), ..., (x_n, y_n))$

       \STATE {\bfseries Step 2: Post-processing}
       \WHILE{$\exists C \in \mathcal{C}$ s.t. $|\frac{1}{n}\sum_{i = 1}^n(p(x_i) - y_i)\one(x_i \in C)| > \varepsilon$}
            \STATE $v_C := \mathsf{sgn}\left(\sum_{i = 1}^n (p(x_i) - y_i)\one(x_i \in C)\right)$
            \FORALL{$x \in C$}
                \STATE $p(x) \leftarrow p(x) - v_C\epsilon$
                \STATE $p(x) \leftarrow \max \{0, \min\{p(x), 1\}\}$
            \ENDFOR
        \ENDWHILE

        \STATE {\bfseries Return: }$p$
    \end{algorithmic}
\end{algorithm}
Algorithm~\ref{alg:ma-empirical} uses an iterative auditing approach to bring the predictor closer to being multiaccurate with each iteration of the While loop in Step 2. This is similar to standard algorithms for multiaccuracy presented in the 
literature~\citep{hebert2018multicalibration, kim2019multiaccuracy}, but is differentiated in that it audits using \emph{the entire dataset} at each iteration, rather than sampling fresh data for each update step. This alteration is necessary in order to guarantee we achieve empirical
multiaccuracy.

It follows immediately from the definition of Algorithm~\ref{alg:ma-empirical} that it outputs an empirically multiaccurate predictor (See Lemma~\ref{lem:stopping-condition}). We can also show that the algorithm is guaranteed to terminate in a bounded number of steps, and thus the class of predictors it can output is also bounded, giving us a uniform convergence result that we state formally and prove in Lemma~\ref{lem:uniform-convergence}.

Lemmas~\ref{lem:stopping-condition} and \ref{lem:uniform-convergence} give us the two sufficient conditions for multigroup robustness described in Section~\ref{sec:multi-robust} (empirical multiaccuracy and uniform convergence, respectively). Thus, we can show that $\mathsf{PP}_{\cA}$ satisfies multigroup-robustness and multiaccuracy:

\begin{theorem}
    Let $\cA: (X \times \{0, 1\})^* \rightarrow \cP$ be a deterministic learning algorithm that is guaranteed to output from a finite set of predictors $\cP \subseteq [0, 1]^X$. Let $\mathsf{PP}_{\cA}$ be the algorithm defined by Algorithm~\ref{alg:ma-empirical} on input $\cA$ with input parameter $\epsilon>0$. Then, for any $\delta \in [0, 1]$, $\mathsf{PP}_{\cA}$ satisfies $(\cC, n, (3 + \frac{m}{n})\epsilon, \delta)$-multigroup robustness and $(\cC, n, 3\epsilon, \delta)$-multiaccuracy for any $n \geq \frac{\log(|\cP|(2|\cC|)^{1/\epsilon^2 + 1}/\delta)}{2\epsilon^2}$.
\end{theorem}

\begin{proof}
    The theorem follows immediately by combining Theorem~\ref{thm:general-robustness-and-ma} with Lemmas~\ref{lem:stopping-condition} and \ref{lem:uniform-convergence}.
\end{proof}

\subsection{Loss Minimization Guarantee}

Lastly, we show that the post-processed predictor is not much worse than the initial predictor outputted by the original learning algorithm. Intuitively, this result follows from Lemma~\ref{lem:l2-dec}, which tells us that each iteration of the post-processing step decreases the predictor's empirical $\ell_2$-loss by at least $\epsilon^2$. With this fact in hand, we can appeal to uniform convergence to show that the loss decrease also generalizes to the entire distribution. The following corollary is a general statement that holds for any run of the algorithm and follows from a stronger version of this result included in the appendix (Theorem~\ref{thm:loss-min}). 

\begin{corollary}[Corollary to Theorem~\ref{thm:loss-min}]
    Let $\cA:(X \times\{0, 1\}^*) \rightarrow \cP$ be a deterministic learning algorithm that is guaranteed to output from a finite set of predictors $\cP \subseteq [0, 1]^X$. Let $\mathsf{PP}_{\cA}$ be the algorithm defined by Algorithm~\ref{alg:ma-empirical} on input $\cA$ with parameter $\epsilon > 0$. Given a sample $S = \{(x_1, y_1), ..., (x_n, y_n)\}$ drawn i.i.d. from some distribution $\cD$ over $X \times \{0, 1\}$, let $p := \cA(S)$ be the predictor output by $\cA$ on $S$, and let $p_{\mathsf{PP}} := \mathsf{PP}_{\cA}$ be the predictor output after post-processing. Then, for any $\delta \in [0, 1]$ and $n \geq \frac{\log(|\cP|(2|\cC|)^{1/\epsilon^2})}{2\epsilon^2}$, we are guaranteed that with probability at least $1 - \delta$
    \[\bE_{\cD}[(y - p_{\mathsf{PP}}(x))^2] \leq \bE_{\cD}[(y - p(x))^2] + 2\epsilon.\]
\end{corollary}

\begin{figure*}[t]
    \includegraphics[width=\textwidth]{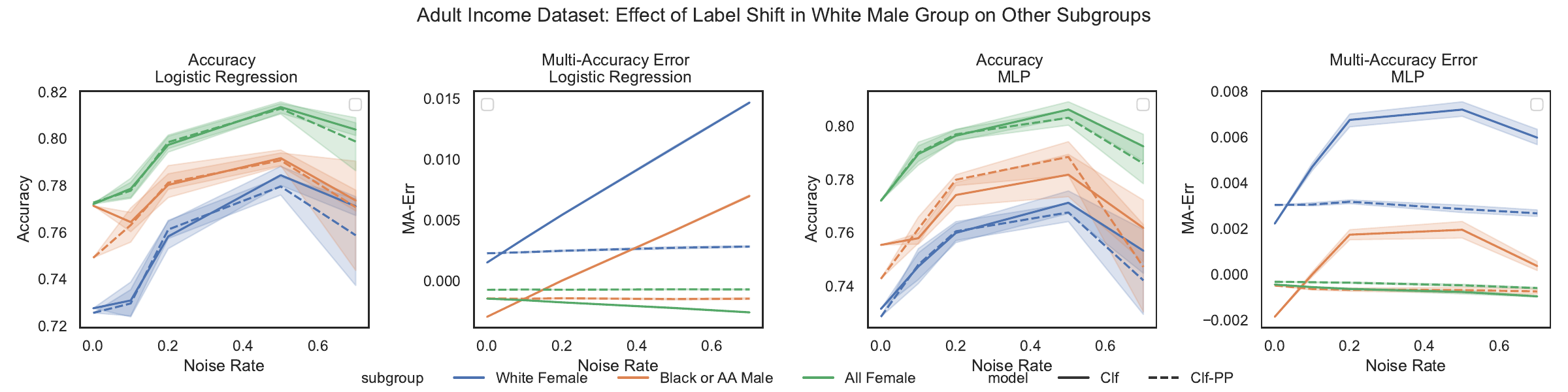}
    \caption{\small The effect of label change (0 to 1) in White male group on other subpopulations. For $\mae$ (closer to 0 is better),  the base models (\textsc{Clf}) are susceptible to noise other groups, Algorithm \ref{alg:ma-empirical} produces multigroup robust predictors (\textsc{Clf-PP}).}
    \label{fig:type1-label-shift}
\end{figure*}
\begin{figure*}
    \centering
    \includegraphics[width=\textwidth]{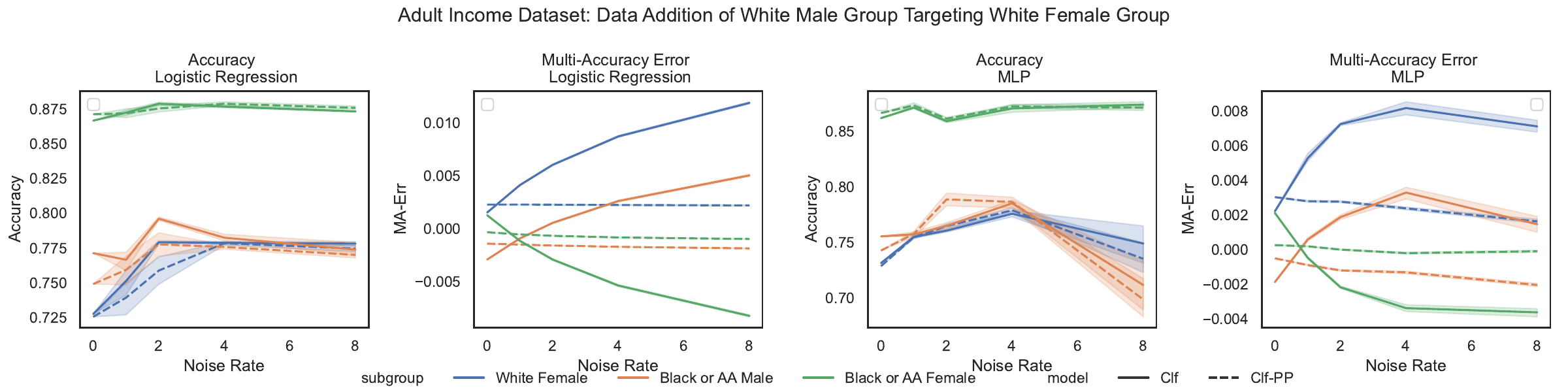}
    \caption{\small The effect of targeting the White female subgroup when only data addition from the White male subgroup is allowed. Multigroup robust predictors (\textsc{Clf-PP}) maintain a consistently low $\mae$ and high accuracy as more corrupted data is injected.}
    \label{fig:type2-addition}
\end{figure*}
\subsection{Additional Extensions and Remarks}

\paragraph{A data-efficient alternative algorithm.} Previous algorithms for multicalibration require fresh data in each ``boosting'' iteration or apply adaptive data analysis techniques to ensure generalization~\citep{hebert2018multicalibration, kim2019multiaccuracy}.  
In contrast, \Cref{thm:general-robustness-and-ma} shows that the iterations can instead simply reuse the same data as in Algorithm~\ref{alg:ma-empirical}, as long as the final model belongs to a class with bounded complexity. This property is often guaranteed for multicalibration algorithms (see e.g., \Cref{lem:uniform-convergence}).\footnote{While our results are stated for multiaccuracy, they can be applied to multicalibration with appropriate modifications.} At a high level, previous multicalibration algorithms are analogous to the original boosting algorithm \citep{weak}, whereas ours are analogous to AdaBoost algorithm \citep{adaboost}. Using this alternative approach is crucial for our robustness guarantees.

\paragraph{Omniprediction.} Recent work has explored the concept of an \emph{omnipredictor}--a predictor whose predictions can be post-processed to select near-optimal actions for a large variety of loss functions, rather than needing to train the predictor to optimize for a specific loss function~\citep{gopalan2022omnipredictors, gopalan2023loss, hu2023omnipredictors, garg2024oracle}. It is known that multicalibration (a strengthening of multiaccuracy) is a sufficient condition for omniprediction. We note that our algorithm can be extended to the setting of multicalibration with appropriate modifications, which would provide a multigroup robust omniprediction algorithm that could provide accuracy guarantees beyond $\ell_2$ loss. However, we note that while such an algorithm would be multigroup robust with respect to the outputted predictor, it is not clear if it would be multigroup robust with respect to the post-processed optimal actions. We leave this as a direction for future work.

\paragraph{Postprocessing on Fresh Data.} While we envision using our post-processing step as part of an end-to-end learning pipeline and thus use the original data during post-processing, in certain settings, the original learning algorithm's training data may be unavailable and the post-processing step might need to use fresh data. In this case, we would continue to preserve accuracy guarantees, but could only guarantee multigroup robustness against strong adversaries with respect to the dataset used for post-processing. In the weaker distribution shift setting, fresh data for post-processing is enough to ensure overall multigroup robustness assuming that the distribution of the post-processing set is the same as that of the training data. 

\section{Experiments}
\label{sec:experiments}

\paragraph{Models and Datasets}
Due to the multigroup focus of our work, we use Folktables-Income ~\citep{ding2021retiring},  a modern version of the UCI Adult Dataset~\footnote{See Appendix C for full experiments on different datasets, different models classes, and detailed data and algorithm descriptions}. For this task, we seek to predict whether the income of individuals was above \$50k and we examine subgroups defined by race and sex. We compare robustness to different attacks in two settings for two models: (1) \textsc{Clf}: Base classifier model: either logistic regression (LR) or a two-layer neural network (MLP), and (2) \textsc{Clf-PP}: Base classifier with post-processing using Algorithm \ref{alg:ma-empirical}. 
We report subgroup $C$ accuracy and multi-accuracy error for each model at different noise rates:  
\begin{itemize}
    \item $\mathsf{Acc}_{\cD}(p, C)$: $\frac{1}{|C|}\sum_{(x_i, y_i) \in C }\one[y_i = \one[p(x_i)>\gamma]]$
    \item $\mae_{\cD}(p, C)$: $\frac{1}{n}\sum_{i = 1}^n(p(x_i) - y_i)\one(x_i \in C)$
\end{itemize}
We seek to measure these quantities since the multi-accuracy error is the supremum of this $\mae_{\cD}(p, C)$ across all identifiable subgroups $C \in \cC$ and $\mathsf{Acc}_{\cD}(p, C)$ gives an average classification metric where the $\gamma$ threshold is optimized on the entire held-out validation set.

\paragraph{Label-Change}
We first consider robustness to label change when the set of training examples is unchanged. The noise rate represents the proportion of data points in the modified subgroup that has been shifted from 0 to 1\footnote{Full algorithm for label change poisoning is in Appendix C}. Figure \ref{fig:type1-label-shift} shows the effect of randomly shifting labels in the White male group on three other groups that are unchanged: White female, Black or African American Male and Female subgroups. For Logistic Regression, the base classifier \textsc{Clf}, results in increased bias as the noise rate increases while the \textsc{Clf-PP} predictor retains a low multi-accuracy error while maintaining similar accuracy to the original model. A similar phenomenon is observed in the neural network (MLP) where \textsc{Clf-PP} remains multigroup robust. 

\paragraph{Addition/Deletion}
We also demonstrate the robustness of Algorithm \ref{alg:ma-empirical} through to additional data points designed to attack a specific target group. We employ a similar strategy as prior work designing subpopulation attacks~\citep{jagielski2021subpopulation} with the additional constraint that only data points outside of the target group can be used to create the poisoning dataset. To find data points to add that would affect the target group, we first cluster data points in a held-out set (not used for training or testing) using K-means. For each cluster where the target subgroup appears, we shift the labels of the of points we are allowed to modify and add them to the training data. The amount of noise in this attack is scaled by how many times the identified data points are replicated before being added to the poisoned dataset. In Figure \ref{fig:type2-addition}, we see that even at low levels of noise (i.e., poisoned points are replicated once or twice), the Logistic Regression and Neural Network (MLP) classifiers exhibit worsened multi-accuracy error while their post-processed counterparts (\textsc{Clf-PP}) exhibit consistently low $\mae$ without lower accuracy. 

\section{Discussion and Future Work}
Motivated by practical scenarios where subgroups in datasets may be corrupted, we present \emph{multigroup robustness} and provide an algorithm that gives meaningful robustness guarantees. Moreover, we empirically show that while standard models allow unrelated groups to suffer under data poisoning attacks, our algorithm applied to post-processing these predictors using the same poisoned data achieves multigroup robustness. While our analysis was limited to the binary setting, notions of multigroup robustness in the multi-class setting are an exciting direction for future work.

\bibliography{ref}

\appendix
\section{Detailed Related Work}
\paragraph{Fairness and influential Datapoints}
The fairness outcomes of machine learning models can be attributed to training data \cite{roh2021sample}. For example, demographic parity and equality of opportunity can depend on very few instances in the dataset according to influence scores. As a result, simply removing highly influential points can yield a dataset that results in models that are Pareto optimal for error and fairness metrics \citep{sattigeri2022fair}. More generally beyond just fairness metrics, \citet{ilyas2022datamodels} present a framework for prediction model prediction based on different subsets of the training data. 

\paragraph{Fairness and Data Poisoning}
Prior works have demonstrated that the fairness properties of machine learning models can be degraded by modifying small subsets of the training data~\cite{solans2020poisoning, van2022poisoning, chai2023robust}. Furthermore, \citet{jagielski2021subpopulation} shows empirically that subpopulations can be directly targeted in data poisoning attacks while minimizing the impact on non-target subpopulations. Algorithms for finding fair predictors that are robust to data poisoning to any part of the dataset have also been proposed for regression~\cite{jin2023fairness}. The term ``subgroup robustness" has also appeared in recent literature to mean the performance of the worst, often intersectional group \cite{martinez2021blind, gardner2022subgroup}. In the context of distribution shift, ``subgroup robustness" is used interchangeably with ``worst group robustness" \citep{sagawa2019distributionally}. In contrast, our definitions focus on characterizing how modifications to training data, due to adversaries or sampling biases, impact different subgroups. 

\paragraph{Multiaccuracy and Multicalibration}
Multiaccuracy and multicalibration are multigroup fairness notions introduced by 
\citet{hebert2018multicalibration} \citep[see also][]{gerrymander,tradeoff}. These notions require a predictor to provide meaningful statistical guarantees (e.g., accuracy in expectation, calibration) on a large family of possibly overlapping subgroups of a population. 
\citet{kim2019multiaccuracy} apply post-processing to neural network models to achieve multiaccuracy.
Recently, \citet{lossmc} show that minimizing a proper loss over neural networks of a certain size yields multicalibration w.r.t.\ all subgroups identifiable by neural nets of a smaller size. 

\citet{uniadapt} show that multicalibration can ensure a predictor's robustness against distribution shift, achieving \emph{universal adaptability}. They focus on covariate shift, where the marginal distribution of $x$ may change between training and testing, but the conditional distribution of $y$ given $x$ remains the same. Their results assume that the covariate shift can be represented by a propensity score function from a given class. Our work considers general forms of data corruption, both in covariate $x$ and label $y$, and the corrupted data need not be i.i.d.\ from any distribution.

\section{Missing Proofs}

\subsection{Section 5}\label{sec:sec5-proofs}

\begin{proof}[Proof of Lemma~\ref{lem:ma-rob}]
    The proof follows by invoking the definition of multiaccuracy. Consider any $C \in \cC$. Because $p$ and $p'$ are both $\varepsilon$-MA with respect to $\cD$ and $\cD'$, respectively, we are guaranteed that 
    \begin{align}
        \left| \mae_{\cD}(p, C) - \mae_{\cD'}(p', C)\right| \leq 2\varepsilon
    \end{align}

    Expanding out the definition of $\mae$ gives us 
    \begin{align}
        &\mae_{\cD}(p, C) - \mae_{\cD'}(p', C) \\
        &= \mae_{\cD}(p, C) - \mae_{\cD'}(p', C)\\
        &\quad+ \bE_{(x, y) \sim \cD}[p'(x)\one[x \in C]]\\
        &\quad- \bE_{(x, y) \sim \cD}[p'(x)\one[x \in C]]\\
        &= \bE_{(x, y) \sim \cD}[(p'(x) - p(x))\one[x \in C]]\label{eq:ma-total-1}\\
        &\quad+ \mae_{\cD}(p', C) - \mae_{\cD'}(p', C)\label{eq:ma-total-2}
    \end{align}

    And thus we are guaranteed that the absolute value of the equation spanning lines \ref{eq:ma-total-1} and \ref{eq:ma-total-2} is at most $2\varepsilon$. By moving the terms in line~\ref{eq:ma-total-2} to the right-hand-side, we get the lemma's statement:
        \begin{align}
        &\left|\bE_{\cD}[(p(x) - p'(x))\one[x \in C]]\right| \\&\leq
        \left| \mae_{\cD}(p', C) - \mae_{\cD'}(p', C)\right|+ 2\varepsilon.
    \end{align}

    This completes the proof.
\end{proof}

\begin{proof}[Proof of Lemma~\ref{lem:ma-robust-dist-shift}]
    For an arbitrary predictor $p'$ and $C \in \cC$, we expand the quantity in the upper bound of Lemma~\ref{eq:ma-lemma}:

    \begin{align*}
        &\mae_{\cD}(p', C) - \mae_{\cD'}(p', C) \\
        &= \bE_{\cD}[(y - p'(x))\one[x \in C]] - \bE_{\cD'}[(y - p'(x))\one[x \in C]]\\
        &= \bE_{\cD}[y\one[x \in C]] - \bE_{\cD'}[y\one[x \in C]]\\
        &\quad+ \bE_{\cD'_X}[p'(x)\one[x \in C]] - \bE_{\cD_X}[p'(x)\one[x \in C]]
    \end{align*}

    Where we note that we can upper bound the quantity on the last line by 
    \begin{align*}
        &\bE_{\cD'_X}[p'(x)\one[x \in C]] - \bE_{\cD_X}[p'(x)\one[x \in C]]\\
        &= \sum_{x \in C}\left(\Pr_{X \sim \cD'_X}[X = x] - \Pr_{X \sim \cD_X}[X = x]\right)p'(x)\\
        &\leq \sum_{x \in C}|\Pr_{X \sim \cD'_X}[X = x] - \Pr_{X \sim \cD_X}[X = x]|\\
        &= \Delta_C(\cD_X, \cD'_X).
    \end{align*}

    Therefore,
    \begin{align*}
        &\left|\mae_{\cD}(p', C) - \mae_{\cD'}(p', C)\right| \\
        &\leq \left|\bE_{\cD}[y\one[x \in C]] - \bE_{\cD'}[y\one[x \in C]]\right| + \Delta_C(\cD_X, \cD'_X)
    \end{align*}

    Thus, whenever $p$ and $p'$ are both $\varepsilon$-MA, we can apply Lemma~\ref{lem:ma-rob} to conclude that 
    \begin{equation}
    \label{eq:dist-shift-req}
    \begin{aligned}
        &\left|\bE_{\cD}[(p(x) - p'(x))\one[x \in C]]\right| \\
        &\leq \left|\bE_{\cD}[y\one[x \in C]] - \bE_{\cD'}[y\one[x \in C]]\right| \\
        &+ \Delta_C(\cD_X, \cD'_X) + 2\varepsilon.
    \end{aligned}
    \end{equation}
    
    It remains to show that this happens with probability at least $1 - 2\delta$ over the randomness of the samples from $\cD$ and $\cD'$. This follows from observing that the probability that $\cA$ fails to output a multiaccurate predictor is at most $\delta$, and so the probability $\cA$ outputs a $p$ or $p'$ that is not multiaccurate is at most $2\delta$ by a union bound. We conclude that equation~\ref{eq:dist-shift-req} holds with probability at least $1 - 2\delta$ over the random samples from $\cD$ and $\cD'$, completing the proof.
\end{proof}

\begin{proof}[Proof of Lemma~\ref{lem:ma-ptwise-rob}]
    For ease of notation, we denote $\bE_{(x,y) \sim \mathsf{Uni}(S)}$ and $\bE_{(x, y) \sim \mathsf{Uni}(S')}$ by $\bE_{S}$ and $\bE_{S'}$ respectively. 

    Consider any $C \in \cC$. We rewrite the LHS of equation~\ref{eq:pointwise-rob-lem} as 
    \begin{align*}
        &\left|\bE_{S}[(p'(x) - p(x))\one[x \in C]]\right|\\
        &= \left|\bE_{S}[(y - p(x))\one[x \in C]]\right| + \left|\bE_{S}[(p'(x) - y)\one[x \in C]]\right|\\
        & \leq \left|\bE_{S}[(p'(x) - y)\one[x \in C]]\right| + \varepsilon
    \end{align*}
    where the last line invokes the assumption that $p$ is $(\cC, \varepsilon)$-MA with respect to $\mathsf{Uni}(S)$. Moreover, we have
    \begin{align*}
        &\left|\bE_{S}[(p'(x) - y)\one[x \in C]]\right| + \varepsilon \\
        &\leq \left|\bE_{S}[(p'(x) - y)\one[x \in C]] - \frac{m}{n}\bE_{S'}[(p'(x) - y)\one[x \in C]]\right| \\
        &\quad + \frac{m}{n}\left|\bE_{S'}[(p'(x) - y)\one[x \in C]]\right| + \varepsilon \\
        &\leq \left|\bE_{S}[p'(x)\one[x \in C]] - \frac{m}{n}\bE_{S'}[p'(x)\one[x \in C]]\right| \\
        &\quad + \left|\bE_{S}[y\one[x \in C]] - \frac{m}{n}\bE_{S'}[y\one[x \in C]]\right|  + \frac{m}{n}\varepsilon' + \varepsilon 
    \end{align*}
    Where this time we use the assumption that $p'$ is $(\cC, \varepsilon')$-MA with respect to $\mathsf{Uni}(S').$ Substituting in the definition of the uniform distributions over $S$ and $S'$, we have 
    \begin{align}
        &\left|\bE_{S}[p'(x)\one[x \in C]] - \frac{m}{n}\bE_{S'}[p'(x)\one[x \in C]]\right| \\
        &\quad + \left|\bE_{S}[y\one[x \in C]] - \frac{m}{n}\bE_{S'}[y\one[x \in C]]\right|  + \frac{m}{n}\varepsilon' + \varepsilon \\
        &= \frac{1}{n}\left|\sum_{i = 1}^n p'(x_i)\one[x_i \in C]] - \sum_{j = 1}^m p'(x_j')\one[x_j' \in C]]\right|\label{eq:p'-diff} \\
        &\quad + \frac{1}{n}\left|\sum_{i = 1}^n y_i\one[x_i \in C]] - \sum_{j = 1}^m y_j'\one[x_j' \in C]]\right|  + \frac{m}{n}\varepsilon' + \varepsilon 
    \end{align}

    Thus, it suffices to show that \ref{eq:p'-diff} is upper-bounded by $\frac{1}{n}|(S \Delta_X S') \cap C|$. Indeed, we can rewrite in terms of maps $\mu_S$ and $\mu_S'$ for the multisets $\{x_1, ..., x_n\}$ and $\{x_1', ..., x_m'\}$ to get 
    
    \begin{align*}
        &\frac{1}{n}\left|\sum_{i = 1}^n p'(x_i)\one[x_i \in C]] - \sum_{j = 1}^m p'(x_j')\one[x_j' \in C]]\right| \\
        &= \frac{1}{n}\left|\sum_{x \in X}(\mu_S(x) - \mu_{S'}(x))p'(x)\one[x \in C]\right|\\
        &\leq \frac{1}{n}\left|\sum_{x \in X}|\mu_S(x) - \mu_{S'}(x)|\one[x \in C]\right|\\
        &= \frac{1}{n}|(S \Delta_X S') \cap C|.
    \end{align*}

    This completes the proof. 
\end{proof}

\begin{proof}[Proof of Theorem~\ref{thm:general-robustness-and-ma}]
    We restate the theorem as two separate lemmas (\ref{thm:general-robustness} and \ref{thm:general-ma}) and provide proofs for each below. The result immediately follows from the combination of these lemmas. 
\end{proof}

\begin{lemma}\label{thm:general-robustness}
Let $\cA: (X \times \{0, 1\})^* \rightarrow \cP$ be a deterministic learning algorithm that outputs predictors from the family $\cP \subseteq [0,1]^X$. For a family of subpopulations $\cC \subseteq 2^X$, suppose that $\cA$ satisfies empirical $(\cC, \varepsilon_1)$-multiaccuracy
and additionally suppose that $\cP$ satisfies $(\cC, n, \varepsilon_2, \delta_2)$-uniform convergence. Then, $\cA$ is $(\cC, n, \left(1 + \frac{m}{n}\right)\varepsilon_1 + 2\varepsilon_2, \delta_2)$-multigroup robust.
\end{lemma}

\begin{lemma}\label{thm:general-ma}
Let $\cA: (X \times \{0, 1\})^* \rightarrow \cP$ be a deterministic learning algorithm that outputs predictors from the family $\cP \subseteq [0,1]^X$. For a family of subpopulations $\cC \subseteq 2^X$, suppose that $\cA$ satisfies empirical $(\cC, \varepsilon_1)$-multiaccuracy and additionally suppose that $\cP$ satisfies $(\cC, n, \varepsilon_2, \delta_2)$-uniform convergence. Then for any distribution $\cD$ over $X \times \{0, 1\}$, $\cA$ is a $(\cC, n, \varepsilon_1 + 2\varepsilon_2, \delta_2)$-multiaccurate learning algorithm. 
\end{lemma}

\begin{proof}[Proof of Lemma~\ref{thm:general-robustness}]
Take any $C \in \cC$ and distribution $\cD_X$. Let $x_1, ..., x_n$ be points drawn i.i.d. from $\cD_X$, and let $S = (x_1, y_1), ..., (x_n, y_n)$, $S' = (x_1', y_1'), ..., (x_m', y_m')$ be datasets where the points other than $x_1, ..., x_n$ can be any value. 

We consider the LHS of the multigroup robustness criterion:
\begin{align}
&\left|\bE_{D_X}[(p(x) - p'(x))\one[x \in C]]\right| \\
&\leq \left|\bE_{D_X}[(p(x) - p'(x))\one[x \in C]] - \bE_S[(p(x) - p'(x))\one[x \in C]] \right|\\
&\quad+ \left|\bE_S[(p(x) - p'(x))\one[x \in C]] \right|\label{eq:emp-distance}\\
\end{align}

Due to our assumption of empirical multiaccuracy, we know that $p$ and $p'$ are $(\cC, \varepsilon_1)$-MA predictors with respect to $\mathsf{Uni}(S)$ and $\mathsf{Uni}(S')$, respectively. Thus, we can apply Lemma~\ref{lem:ma-ptwise-rob} to upper bound \ref{eq:emp-distance} to get
\begin{align}
&\left|\bE_{D_X}[(p(x) - p'(x))\one[x \in C]] - \bE_S[(p(x) - p'(x))\one[x \in C]] \right|\\
&\quad+ \left|\bE_S[(p(x) - p'(x))\one[x \in C]] \right|\\
&\leq \left|\bE_{D_X}[(p(x) - p'(x))\one[x \in C]] - \bE_S[(p(x) - p'(x))\one[x \in C]] \right|\label{eq:pred-uc-dist}\\
&\quad+ \frac{1}{n}\left|\sum_{i = 1}^ny_i \one[x_i \in C] - \sum_{j = 1}^my_j'\one[x_j' \in C]\right|\\
&\quad + \frac{1}{n}|(S \Delta_X S')\cap C| + \left(1 + \frac{m}{n}\right)\varepsilon_1.\\
\end{align}

Finally, we use our uniform convergence assumption to upper-bound \ref{eq:pred-uc-dist}:

\begin{align*}
    &\left|\bE_{D_X}[(p(x) - p'(x))\one[x \in C]] - \bE_S[(p(x) - p'(x))\one[x \in C]] \right|\\
    &\leq \left|\bE_{D_X}[p(x)\one[x \in C]] - \bE_S[p(x)\one[x \in C]] \right| \\
    &\quad+ \left|\bE_{D_X}[p'(x)\one[x \in C]] - \bE_S[p'(x)\one[x \in C]] \right|\\
    &\leq 2\varepsilon_2
\end{align*}

where the final bound follows from uniform convergence and holds simultaneously for all $C \in \cC$ with probability at least $1 - \delta_2$. 

Substituting this into our original bound, we conclude that with probability at least $1 - \delta_2$, for all $C \in \cC$, $p$ and $p'$ satisfy 

\begin{align*}
    &\left|\bE_{D_X}[(p(x) - p'(x))\one[x \in C]]\right|\\
    &\leq \frac{1}{n}\left|\sum_{i = 1}^ny_i \one[x_i \in C] - \sum_{j = 1}^my_j'\one[x_j' \in C]\right|\\
&\quad + \frac{1}{n}|(S \Delta_X S')\cap C| + \left(1 + \frac{m}{n}\right)\varepsilon_1 + 2\varepsilon_2. 
\end{align*}

Thus, we conclude that $\cA$ is $(\cC, n, \left(1 + \frac{m}{n}\right)\varepsilon_1 + 2\varepsilon_2, \delta_2)$-multigroup robust. 
\end{proof}

\begin{proof}[Proof of Lemma~\ref{thm:general-ma}]
Consider any $\cD$ and let $p = \cA((x_1, y_1), ..., (x_n, y_n))$ where each $(x_i, y_i)$ is drawn i.i.d. from $\cD$ for sufficiently large $n \geq m(\cC, \varepsilon_2, \delta_2)$. 

Consider any $C \in \cC$. We note that 
\begin{align*}
    &\left|\bE_{D}[(y - p(x))\one[x \in C]]\right| \\
    &= \left|\bE_{D}[(y - p(x))\one[x \in C]] + \frac{1}{n}\sum_{i = 1}^n(y_i - p(x_i))\one[x_i \in C] - \frac{1}{n}\sum_{i = 1}^n(y_i - p(x_i))\one[x_i \in C]\right|\\
    &\leq \left|\bE_{D}[(y - p(x))\one[x \in C]] - \frac{1}{n}\sum_{i = 1}^n(y_i - p(x_i))\one[x_i \in C]\right| + \left|\frac{1}{n}\sum_{i = 1}^n(y_i - p(x_i))\one[x_i \in C]\right| \\
    &\leq \left|\bE_{D}[(y - p(x))\one[x \in C]] - \frac{1}{n}\sum_{i = 1}^n(y_i - p(x_i))\one[x_i \in C]\right| + \varepsilon_1 
\end{align*}
Where the last step holds by our assumption that $p$ is empirically multiaccurate. Continuing to simplify, we get 

\begin{align*}
    &\left|\bE_{D}[(y - p(x))\one[x \in C]] - \frac{1}{n}\sum_{i = 1}^n(y_i - p(x_i))\one[x_i \in C]\right| + \varepsilon_1  \\
    &\leq \left|\bE_{D}[p(x)\one[x \in C]] - \frac{1}{n}\sum_{i = 1}^n p(x_i)\one[x_i \in C]\right| + \left|\bE_{D}[y\one[x \in C]] - \frac{1}{n}\sum_{i = 1}^ny_i\one[x_i \in C]\right| + \varepsilon_1 
\end{align*}

By our uniform convergence guarantee, we are guaranteed that the above quantity is bounded by $2\varepsilon_2 + \varepsilon_1$ simultaneously for all $\cC$ with probability at least $1 - \delta_2$, and thus we know that with probability at least $1 - \delta_2$, $p$ satisfies 
\[\left|\bE_{D}[(y - p(x))\one[x \in C]]\right| \leq 2\varepsilon_2 + \varepsilon_1\]
for all $C \in \cC$, and thus is $(\cC, 2\varepsilon_2 + \varepsilon_1)$-MA with probability at least $1- \delta_2$, completing the proof. 
\end{proof}

\subsection{Section 6}\label{sec:sec6-proofs}

\begin{lemma}[Empirical MA of Algorithm~\ref{alg:ma-empirical}]\label{lem:stopping-condition}
    Algorithm~\ref{alg:ma-empirical} satisfies empirical $(\cC, \varepsilon)$-multiaccuracy.
\end{lemma}

\begin{proof}
    Note that the stopping condition of Algorithm~\ref{alg:ma-empirical} implies that for all $C \in \cC$, 
    \[\left|\frac{1}{n}\sum_{i = 1}^n (y_i - p(x_i))\one[x_i \in C]\right| \leq \varepsilon.\]
    Thus, $p$ must be empirically $(\cC, \varepsilon)$-MA. 
\end{proof}

While we have verified that the algorithm will satisfy empirical multiaccuracy upon termination, we have yet to verify whether it satisfies the uniform
convergence property, which together with Lemma~\ref{lem:stopping-condition} would imply multigroup robustness by our result in Theorem~\ref{thm:general-robustness-and-ma}. Moreover, we need to show that the post-processing
step does not decrease the predictor's $\ell_2$ error by too much, and also that the algorithm is actually guaranteed to terminate in a small number of steps. All of these properties will follow from the following lemma, which shows that each update 
step can \emph{only decrease} the predictor's $\ell_2$ error on the empirical distribution. 

\begin{lemma}\label{lem:l2-dec}
   Given an arbitrary dataset $S = \{(x_1, y_1), ..., (x_n, y_n)\}$, let $\nu_S: [0, 1]^X \rightarrow \bR^{\geq 0}$ be a function capturing the empirical $\ell_2$ error of a predictor on $S$:
   \[\nu_S(p) = \frac{1}{n}\sum_{i = 1}^n(y_i - p(x_i))^2.\]  
   Given input data $S$, every iteration of the while loop in Algorithm~\ref{alg:ma-empirical} strictly decreases $\nu_S$ by at least $\epsilon^2$, or terminates.
\end{lemma}

\begin{proof}[Proof of Lemma~\ref{lem:l2-dec}]
    Let $p \in [0, 1]^X$ be an arbitrary predictor and consider one iteration of the while loop in Algorithm~\ref{alg:ma-empirical}. 

    If there exists no $C \in \cC$ with $\left|\frac{1}{n}\sum_{i = 1}^n (p(x_i) - y_i)\one[x_i \in C] \right| > \varepsilon$, then the algorithm terminates by definition. 

    In the other case, there exists some $C \in \cC$ with
    $\left|\frac{1}{n}\sum_{i = 1}^n (p(x_i) - y_i)\one[x_i \in C] \right| > \varepsilon$. Let $v_C = \mathsf{sgn}(\frac{1}{n}\sum_{i = 1}^n (p(x_i) - y_i)\one[x_i \in C])$, and let $p'$ be the updated predictor such that 
    \[p'(x) = p(x) - v_C\varepsilon\one[x \in C].\]

    We consider the difference in potential functions:
    \begin{align*}
        \nu_S(p) - \nu_S(p') &= \frac{1}{n}\sum_{i = 1}^n (p(x_i) - y_i)^2 - (p'(x_i) - y_i)^2
    \end{align*}
    we note that for all $x_i \not\in C$, $p(x_i)= p'(x_i)$, so we can cancel out these terms to get 
    \begin{align*}
        &\frac{1}{n}\sum_{x_i, y_i \in S, x_i \in C} (p(x_i) - y_i)^2 - (p'(x_i) - y_i)^2\\
        &= \frac{1}{n}\sum_{x_i, y_i \in S, x_i \in C} (p(x_i) - y_i + p'(x_i) - y_i)(p(x_i) - y_i - p'(x_i) + y_i)\\
        &= \frac{1}{n}\sum_{x_i, y_i \in S, x_i \in C} (2(p(x_i) - y_i) - v_C\epsilon)v_C\epsilon\\
        & = 2\epsilon\left|\frac{1}{n}\sum_{i = 1}^n(p(x_i) - y_i)\one[x_i \in C]\right|- \frac{\left|\{x_i,y_i \in S, x_i \in C\}\right|\epsilon^2}{n} \\
        &\geq 2\epsilon^2 - \frac{\left|\{x_i,y_i \in S, x_i \in C\}\right|\epsilon^2}{n}\\
         &\geq 2\epsilon^2 - \epsilon^2\\
         &= \epsilon^2
    \end{align*}

    so, we have shown that $\nu_S(p) - \nu_S(p') \geq \epsilon^2$. Note that there is a final update 
    to clip $p'$ to be $[0, 1]$ values on all $x$. However, because every $y_i \in [0, 1]$, this can only reduce the $\ell_2$ error ($\nu_S$) of the updated predictor, and so we conclude that every non-terminating iteration of the algorithm reduces $\nu_S$ by at least $\epsilon^2$, proving the claim.
\end{proof}

As promised, a number of nice properties can now be derived using Lemma~\ref{lem:l2-dec}. First, we can bound the number of iterations of the algorithm via a potential function argument (See the appendix~\ref{sec:sec6-proofs} for details):

\begin{lemma}[Stopping Time of Algorithm~\ref{alg:ma-empirical}]\label{lem:ma-emp-stopping-time}
    Algorithm~\ref{alg:ma-empirical} makes at most $1/\epsilon^2$ iterations of the while loop before terminating. 
\end{lemma}

\begin{proof}[Proof of Lemma~\ref{lem:ma-emp-stopping-time}]
    Given input dataset $S = \{(x_1, y_1), ..., (x_n, y_n)\}$, we use the empirical $\ell_2$ error, 
    \[\nu_S(p) = \frac{1}{n}\sum_{i = 1}^n(y_i - p(x_i))^2\]

    as a potential function. We note that by definition, for any $p$, we are guaranteed that $0 \leq \nu_S(p) \leq 1$. 

    Thus, applying the result of Lemma~\ref{lem:l2-dec}, we can conclude that the number of iterations without termination can be at most $1/\epsilon^2$, otherwise $\nu_S$ would become negative, resulting in a contradiction. 
\end{proof}

Next, we show a uniform convergence result by bounding the set of predictors that can be outputted by Algorithm~\ref{alg:ma-empirical}. Intuitively, we do this by using the result of Lemma~\ref{lem:ma-emp-stopping-time} to conclude that not too many updates are made to the predictor during post-processing, and thus the class of predictors output by the post-processing step is not too much more complex than that of the original learning algorithm. 

\begin{lemma}\label{lem:uniform-convergence}
    Let $\cA: (X \times \{0, 1\})^* \rightarrow \cP$ be a deterministic learning algorithm that is guaranteed to output from a finite set of predictors $\cP \subseteq [0, 1]^X$. Let $\mathsf{PP}_{\cA}$ be the algorithm defined by Algorithm~\ref{alg:ma-empirical} on input $\cA$ and $\epsilon > 0$. Then, for any $\delta \in [0, 1]$ the family of predictors that can be output by $\mathsf{PP}_{\cA}$ satisfies $(\cC, n, \epsilon, \delta)$-uniform convergence for any $n \geq \frac{\log(|\cP|(2|\cC|)^{1/\epsilon^2 + 1}/\delta)}{2\epsilon^2}$.
\end{lemma}

\begin{proof}[Proof of Lemma~\ref{lem:uniform-convergence}]
    Let $\cP_{\mathsf{PP}}$ be the family of predictors that can be output by $\mathsf{PP}_{\cA}$. We proceed by bounding the size of $\cP_{\mathsf{PP}}$.

    Note that the predictor output by the initialization step of $\mathsf{PP}_{\cA}$ must be from $\cP$, and so there are $|\cP|$ possible predictors that can be output at the initialization step of the algorithm. 

    From there, for any particular predictor output during initialization, note that at each iteration of the algorithm, there are at most $2|\cC|$ possible different updates that can be made to the predictor. 

    Thus, applying the result of Lemma~\ref{lem:ma-emp-stopping-time} that says the number of iterations is bounded by $1/\epsilon^2$, we can conclude that the number of possible predictors that can be output by the algorithm is at most 
    $\left(|\cP|\right)(2|\cC|)^{\frac{1}{\epsilon^2}}$.

    By Chernoff bounds, for any particular $p$ and $C$, the probability that 
    \[\left|\frac{1}{n}\sum_{i = 1}^n p(x_i)\one[x_i \in C] - \bE_{\cD}[p(x)\one[x \in C]\right| > \epsilon\]
    is upper bounded by $2\exp(-2n\epsilon^2)$. 

    Including the $y$s noted in the definition of uniform convergence, this means that we need the above condition to hold for a total of 
    \[\left(\left(|\cP|\right)(2|\cC|)^{\frac{1}{\epsilon^2}} + 1\right)|\cC|\]
    pairs of $p$ and $C$. 

    Thus, it suffices to take \[n \geq \frac{\log(|\cP|(2|\cC|)^{1/\epsilon^2 + 1}/\delta)}{2\epsilon^2}\]

    and thus $\cP_{\mathsf{PP}}$ satisfies $(\cC, n, \epsilon, \delta)$-uniform convergence for any $n \geq \frac{\log(|\cP|(2|\cC|)^{1/\epsilon^2 + 1}/\delta)}{2\epsilon^2}$.
\end{proof}

\begin{theorem}\label{thm:loss-min}
    Let $\cA:(X \times\{0, 1\}^*) \rightarrow \cP$ be a deterministic learning algorithm that is guaranteed to output from a finite set of predictors $\cP \subseteq [0, 1]^X$. Let $\mathsf{PP}_{\cA}$ be the algorithm defined by Algorithm~\ref{alg:ma-empirical} on input $\cA$ with parameter $\epsilon > 0$. Given a sample $S = \{(x_1, y_1), ..., (x_n, y_n)\}$ drawn i.i.d. from some distribution $\cD$ over $X \times \{0, 1\}$, let $p := \cA(S)$ be the predictor output by $\cA$ on $S$, and let $p_k$ be the predictor after $k \geq 0$ iterations of the post-processing update in $\mathsf{PP}_{\cA}(S)$. Then, for any $\delta \in [0, 1]$ and $n \geq \frac{\log(|\cP|(2|\cC|)^k)}{2\epsilon^2}$, we are guaranteed that with probability at least $1 - \delta$
    \[\bE_{\cD}[(y - p_k(x))^2] \leq \bE_{\cD}[(y - p(x))^2] + 2\epsilon - k\epsilon^2.\]
\end{theorem}

\begin{proof}[Proof of Theorem~\ref{thm:loss-min}]
    Note that by Lemma~\ref{lem:l2-dec}, we are guaranteed that after $k$ iterations, the empirical squared loss has decreased by at least $k\epsilon^2$, and thus
    \[\frac{1}{n}\sum_{i = 1}^n(y_i - p_k(x_i))^2 \leq \frac{1}{n}\sum_{i = 1}^n(y_i - p(x_i))^2 - k\epsilon^2.\]

    Note that by a Chernoff bound, for any fixed predictor $p$, we have that 

    \[\left|\bE_{\cD}[(y - p(x))^2] - \frac{1}{n}\sum_{i = 1}^n(y_i - p(x_i))^2\right| < \epsilon\]

    with probability at most $2\exp(-2n\epsilon^2)$. 

    Using the same reasoning as in Lemma~\ref{lem:uniform-convergence}, both $p$ and $p_k$ are guaranteed to belong to a subset of predictors of size at most 
    $|\cP|(2|\cC|)^k$. 

    Thus, taking $n \geq \frac{\log(|\cP|(2|\cC|)^k)}{2\epsilon^2}$ guarantees that with probability at least $1 - \delta$, we have both
    \[\left|\bE_{\cD}[(y - p(x))^2] - \frac{1}{n}\sum_{i = 1}^n(y_i - p(x_i))^2\right| < \epsilon\]
    and
    \[\left|\bE_{\cD}[(y - p_k(x))^2] - \frac{1}{n}\sum_{i = 1}^n(y_i - p_k(x_i))^2\right| < \epsilon\]

    Thus, substituting these inequalities into our empirical expression gives us
    \[\bE_{\cD}[(y - p_k(x))^2] \leq \bE_{\cD}[(y - p(x))^2] - k\epsilon^2 + 2\epsilon,\]
    completing the proof.
\end{proof}

\section{Lower Bounds}\label{sec:lower-bd-app}

Thus far, we've seen that certain multiaccurate learning algorithms can provide strong multigroup robustness guarantees while preserving performance guarantees when used as a post-processing step on an existing predictor. 

In this section, we explore whether multiaccuracy is a necessary condition of a multigroup robust algorithm under a weak non-triviality assumption. 

Clearly, a learning algorithm that always outputs the same predictor trivially satisfies multigroup robustness while not being multiaccurate, but is not a very useful learning algorithm. In order to circumvent these edge cases, we introduce a very weak non-triviality assumption, namely that an algorithm matches the overall mean of the true outcomes. 

\begin{definition}[Accuracy in Expectation]
We say a learning algorithm $\cA: (X \times \{0, 1\})^* \rightarrow [0, 1]^X$ is $(n, \varepsilon, \delta)$-accurate-in-expectation if given $n$ i.i.d. datapoints $(x_1,y_1), ..., (x_n, y_n)$ from any distribution $\cD$ over $X \times \{0, 1\}$, $\cA$ outputs a predictor $p = \cA((x_1, y_1), ..., (x_n, y_n))$ satisfying 
\[\left|\bE_{D}[y - p(x)]\right| \leq \varepsilon\]
with probability at least $1 - \delta$. 
\end{definition}

Note that this is a far weaker assumption than multiaccuracy, in that the predictor only needs to match the overall mean of the true outcomes, rather than on any particular groups. However, we will show that even this weak accuracy assumption paired with multigroup robustness implies multiaccuracy. 

We restate the theorem for readability:

\begin{theorem}[Lower Bound, Theorem~\ref{thm:lower-bound}]
Let $\cC$ be a class of subpopulations, $\cP \subseteq [0, 1]^X$ a family of predictors containing the all ones predictor $p(x) = 1$ for all $x \in X$, and $\cA: (X \times \{0, 1\}^*) \rightarrow \cP$ a deterministic learning algorithm. If $\cA$ is $(\cC, n, \varepsilon_1, \delta_1)$-multigroup robust and $(n, \varepsilon_2, \delta_2)$-accurate-in-expectation, and $\cP$ satisfies $(\cC, n, \epsilon_3, \delta_3)$-uniform convergence, then $\cA$ is a $(\cC, n,\epsilon_1 + \epsilon_2 + 2 \epsilon_3, 2\delta_1 + 2\delta_2 + \delta_3)$-multiaccurate learning algorithm. 
\end{theorem}

\begin{proof}[Proof of Theorem~\ref{thm:lower-bound}]
    Let $\cA$ be the algorithm described in the theorem that is $(\cC, n, \varepsilon_1, \delta_1)$-multigroup robust and $(n, \varepsilon_2, \delta_2)$-accurate-in-expectation. 

    Consider any $\cD$ over $X \times \{0, 1\}$ with marginal $\cD_X$ and let $S = \{(x_1, y_1), ..., (x_n, y_n)\}$ be drawn i.i.d. from $\cD$. 

    Consider the two amended datasets $S_0 = \{(x_1, 0), ..., (x_n, 0)\}$ and $S_1 = \{(x_1, 1), ..., (x_n, 1)\}$. Let $p_0 = \cA(S_0)$, $p_1 = \cA(S_1)$, $p = \cA(S)$. 

    Because $S_0$ and $S_1$ are identical to i.i.d. samples from distributions with the same marginal distribution but all zeros or all ones, we know that with probability at least $1 - 2\delta_2$, we have
    $$\bE_{D_X}[p_0(x)] \leq \varepsilon_2, \bE_{D_X}[1 - p_1(x)] \leq \varepsilon_2.$$

    When these two inequalities hold true, we also have for all $C \in \cC$:
    \[\bE_{D_X}[p_0(x)\one[x \in C]] \leq \epsilon_2, \]
    \[\bE_{D_X}[p_1(x)\one[x \in C]] \geq \Pr_{D_X}[x \in C] - \epsilon_2.\]

    Now, by multigroup robustness, we are guaranteed that with probability at least $1 - 2\delta_1$, we have the following two inequalities for all $C \in \cC$:

    \[\left|\bE_{D_X}[(p_0(x) - p(x))\one[x \in C]\right| \leq \frac{1}{n}\sum_{i = 1}^n y_i\one[x_i \in C] + \epsilon_1\]
    \[\left|\bE_{D_X}[(p_1(x) - p(x))\one[x \in C]\right| \leq  \frac{1}{n}\sum_{i = 1}^n (1 - y_i)\one[x_i \in C] + \epsilon_1\]

    Simplifying the first inequality gives us 
    \begin{align*}
        &\bE_{D_X}[p(x)\one[x \in C]] \\&\leq \bE_{D_X}[p_0(x)\one[x \in C]] + \frac{1}{n}\sum_{i = 1}^n y_i\one[x_i \in C] + \epsilon_1\\
        &\leq \frac{1}{n}\sum_{i = 1}^n y_i\one[x_i \in C] + \epsilon_1 + \epsilon_2\\
        &\leq \bE_{D}[y\one[x \in C]] + \epsilon_1 + \epsilon_2 + \epsilon_3
    \end{align*}
    Where the last step holds by uniform convergence for all $C \in \cC$ with probability at least $1 - \delta_3$. 

    Similarly, the second inequality simplifies to 
    \begin{align*}
        &\bE_{D_X}[p(x)\one[x \in C]] \\&\geq \bE_{D_X}[p_1(x)\one[x \in C]] +  \frac{1}{n}\sum_{i = 1}^n y_i - \frac{1}{n}\sum_{i = 1}^n\one[x_i \in C] - \epsilon_1\\
        &\geq \Pr_{D_X}[x \in C] +  \frac{1}{n}\sum_{i = 1}^n y_i\one[x_i \in C] - \frac{1}{n}\sum_{i = 1}^n\one[x_i \in C] - \epsilon_1 - \epsilon_2\\
        &\geq \bE_{D}[y\one[x \in C]] - \epsilon_1 - \epsilon_2 - 2\epsilon_3
    \end{align*}
    where the last step follows by two applications of uniform convergence and our assumption that $\cP$ contains the constant ones predictor, and also hold simultaneously for all $C \in \cC$ with probability at least $1 - \delta_3$. 

    Thus, we conclude that with probability at least $1 - (2\delta_1 + 2\delta_2 + \delta_3)$, it holds that
    \[\left|\bE_D[(p(x) - y)\one[x \in C]\right| \leq \epsilon_1 + \epsilon_2 + 2 \epsilon_3\]
    for all $C \in \cC$ and thus $\cA$ is $(\cC, n,\epsilon_1 + \epsilon_2 + 2 \epsilon_3, 2\delta_1 + 2\delta_2 + \delta_3)$-multiaccurate.
\end{proof}
\section{Full Experiments}
\subsection{Models and Datasets}
Due to the fairness focus of our work, we use Folktables ~\cite{ding2021retiring},  a modern version of the UCI Adult Dataset based on the yearly American Community Survey. In this section, we expand on the Income experiments in the main text to other tasks. We look at 3 binary tasks with different feature dimensions, subgroup ratios, and positive class ratios (Table \ref{tab:datasets}). These datasets were taken from a mid-sized state where there is significant racial diversity: Louisiana (the 25th most populous state in the US). For each task, we consider subgroups defined by race and sex. We consider the three most common racial groups as coded by: \textit{White Alone} (62\%), \textit{Black or African American Alone} (32.8\%), and \textit{Asian Alone} (2\%). Although the Hispanic and Latino populations constitute 5.6\% of the population in Louisiana, \footnote{\url{https://www.census.gov/quickfacts/fact/table/LA/RHI725222\#RHI725222}} the Census Bureau officially states that \textit{``People who identify their origin as Hispanic, Latino, or Spanish may be of any race"}\footnote{\url{https://www.census.gov/quickfacts/fact/note/US/RHI625222}}. Thus we include the three aforementioned racial identifiable groups for our experiments.    

\begin{table}
\begin{tabular}{|p{2cm}|p{1.5cm}|p{1.5cm}|p{1.5cm}|p{1.5cm}|p{1.5cm}|p{1.5cm}|p{1.5cm}|l|}
\hline
Dataset             & Task                               & n     & Features & Positive Class Ratio & White Ratio & Black or AA Ratio & Asian Ratio & Male Ratio \\ \hline
ACS Income          & Income $>50k$                      & 20667 & 47            & 0.332                & 0.710       & 0.235             & 0.020       & 0.507      \\ \hline
ACS Employment      & Employed                           & 43589 & 69            & 0.414                & 0.675       & 0.266             & 0.018       & 0.483      \\ \hline
ACS Public Coverage & Covered by Public Health Insurance & 16879 & 76            & 0.406                & 0.589       & 0.344             & 0.023       & 0.424      \\ \hline
\end{tabular}
\caption{Task summary for American Community Survey Dataset for the state of Louisiana 2018 dataset.}
\label{tab:datasets}
\end{table}

For this task, we seek to predict whether the income of individuals was above \$50k and we examine subgroups defined by race and sex. We compare robustness to different attacks in three settings for 5 different models: 
\begin{itemize}
    \item \textsc{MAEmp}: Multiaccuracy Boost on Empirical Distribution (Algorithm \ref{alg:ma-empirical})
    \item \textsc{Clf}: Base classifier model: We evaluate performance across 5 models of different model classes
    \begin{itemize}
        \item LR: Logistic Regression Classifier. Probabilities from this model are probability estimates of each class from the logistic function. 
        \item DT: Decision Tree classifier. Probabilities from this model are the fraction of samples of the same class in a leaf\footnote{\url{https://scikit-learn.org/stable/modules/generated/sklearn.tree.DecisionTreeClassifier.html}}. 
        \item kNN: k-Nearest Neighbors algorithm with parameter search from 3, 5, and 7 nearest neighbors. Probabilities from this model represent the proportion of neighbors in the predicted class.  
        \item XGB: XGBoost algorithm using the default \textit{binary-logistic} objective. Probabilities are class probabilities summed and normalized across trees within the boosting system. 
        \item MLP: Multi-Layer Perceptron neural network with 2 layers after performing hyperparameter search over the learning rate and $l_2$ regularization weight. Probabilities are the output of the logistic function for each class. 
    \end{itemize}
    \item \textsc{Clf-PP}: Base classifier with post-processing using Multiaccuracy Boost on Empirical Distribution. 
\end{itemize}
To evaluate each model, we report subgroup accuracy and multi-accuracy error for each model at different noise rates. More formally, for subgroup $C$:  
\begin{itemize}
    \item $\mathsf{Accuracy}_{\cD}(p, C)$: $\frac{1}{|C|}\sum_{(x_i, y_i) \in C }\one[y_i = \one[p(x_i)>\gamma]]$
    \item $\mae_{\cD}(p, C)$: $\frac{1}{n}\sum_{i = 1}^n(p(x_i) - y_i)\one(x_i \in C)$
\end{itemize}
We seek to measure these quantities since multi-accuracy error is the supremum of this $\mae_{\cD}(p, C)$ across all identifiable subgroups $C \in \cC$ and $\mathsf{Accuracy}$ gives an average classification metric where the $\gamma$ threshold is optimized on a held-out validation set. We measure all of these results on a test set while both training and post-process with Algorithm $\ref{alg:ma-empirical}$ are done on the training set.

\subsection{Label-Change}

\begin{algorithm}
    \caption{Label Change Algorithm}
    \label{alg:label-change}
    \begin{algorithmic}
       \STATE {\bfseries Parameters:} target: $t \in \{0, 1, *\}$, modify group $C \subseteq \{0, 1\}^X$, noise ratio $\sigma \in [0, 1]$. 
       \STATE {\bfseries Input: } clean dataset  $D = (x_1, y_1), ..., (x_n, y_n) \in X \times \{0, 1\}$
       \STATE {\bfseries Output: } corrupted dataset  $D' = (x_1, y_1'), ..., (x_n, y_n') \in X \times \{0, 1\}$
        \STATE $D' = \{\}$ 
        \FORALL{$(x_i, y_i) \in D$}
        \STATE $z \sim Uniform([0, 1])$ 
        \IF{$z < \sigma$ and $y_i == t$}
        \STATE $y_i' = 1 - y_i$
        \ELSE
        \STATE $y_i' = y_i$
        \ENDIF
        \STATE $D' = D' \cup (x_i, y_i')$ 
        \ENDFOR
        \STATE {\bfseries Return: }$D'$
    \end{algorithmic}
\end{algorithm}
\paragraph{Designing Label Shift}
First, we consider robustness to label noise when the set of training examples is unchanged (See Section \ref{sec:multi-robust}). We randomly (1) flip labels, (2) shift labels to 1, and (3) shift labels to 0 within a subgroup and observe the effects on other subgroups. Algorithm \ref{alg:label-change} describes the generation of all three types of label change. The noise rate in this corruption represents the proportion of data points in the modified subgroup that has been shifted.

\begin{figure}
         \centering
         \includegraphics[width=\textwidth]{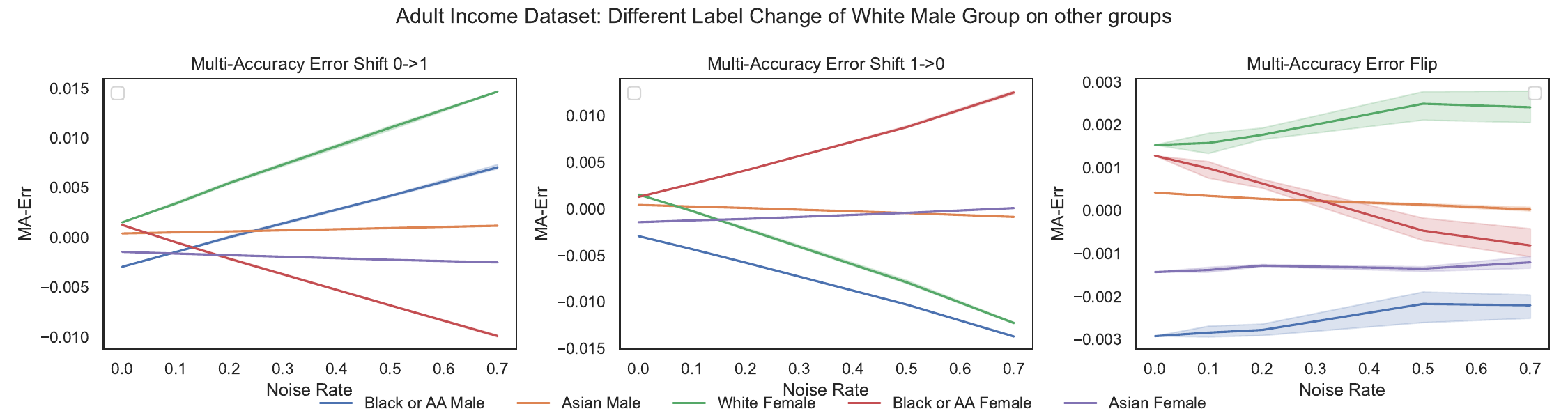}
         \caption{\small Effect of label change in the white male group on other groups with a Logistic Regression Classifier. As the amount of noise increases, the $\mae$ of the resulting predictor grows away from 0. The direction depends on whether the shift in label is from 0 to 1 or from 1 to 0.}
         \label{fig:label-change-white-male}
\end{figure}
 \begin{figure}
     \centering
     \includegraphics[width=\textwidth]{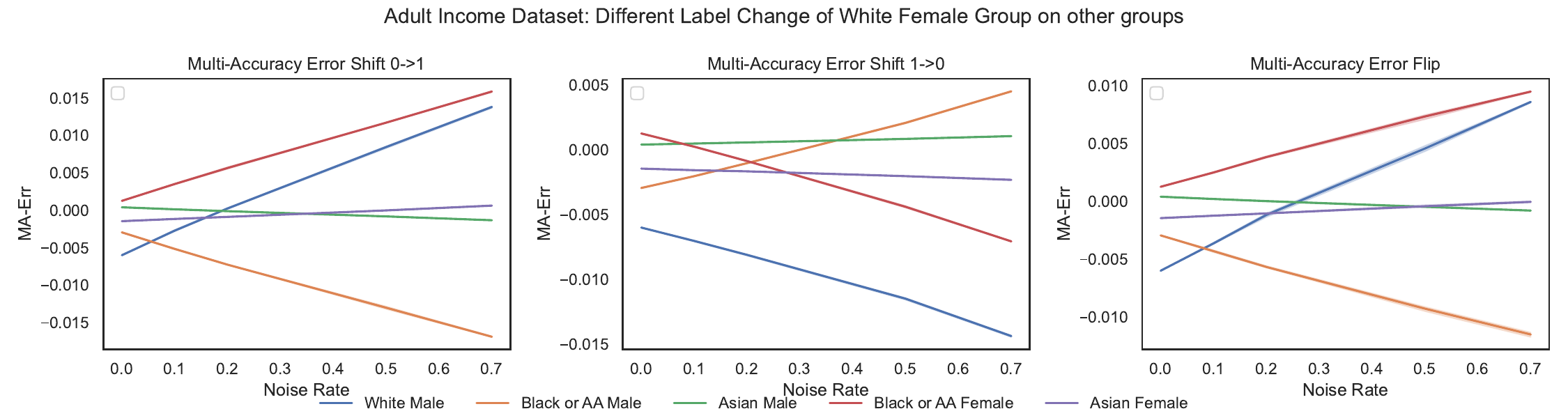}
     \caption{\small Effect of label change in the white female group on other groups with a Logistic Regression Classifier. In this group, flipping the labels has a large effect on the Black population.}
     \label{fig:label-change-white-female}
 \end{figure}
 \begin{figure}
     \centering
     \includegraphics[width=\textwidth]{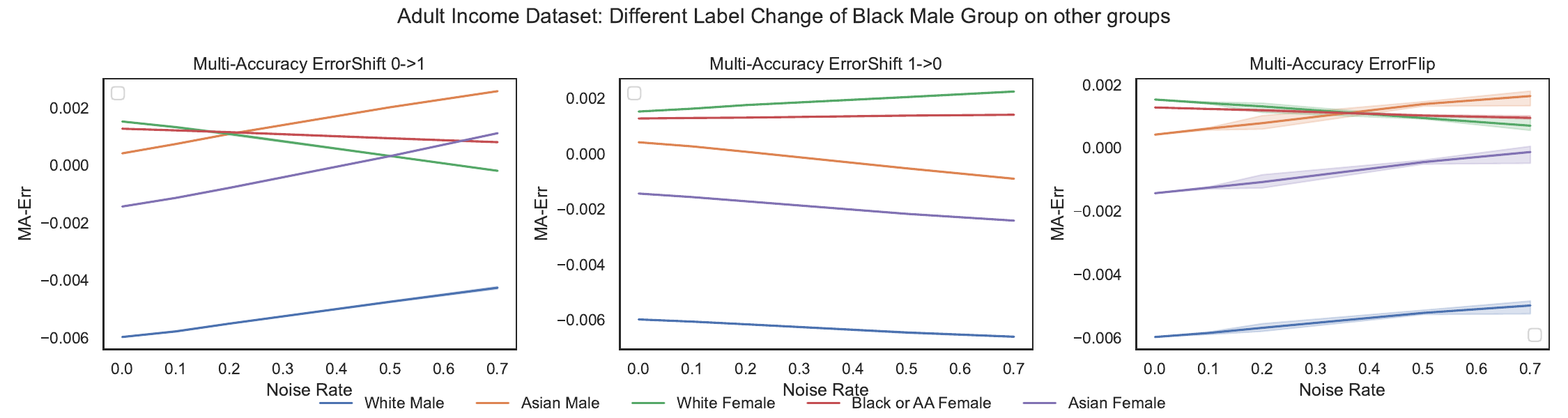}
     \caption{\small Effect of label change in the black male group on other groups with a Logistic Regression Classifier. Since this group is a smaller fraction of the population, the induced change in $\mae$ is much smaller.}
     \label{fig:label-change-black-male}
 \end{figure}
 
We compare the different effects of label shifts first on the Logistic Regression model to understand the impact of label change. Figures \ref{fig:label-change-white-male}, \ref{fig:label-change-white-female}, and \ref{fig:label-change-black-male} show the effect of different types to label change when different groups are modified. We observed significant changes in $\mae$ on other groups particularly when the white male group and white female group are modified by shifting the labels from 0 to 1. We proceed with experiments in label change by modifying the white male group by shifting labels from 0 to 1 in different ratios. These preliminary results also show that due to the low ratio of the Asian population, $\mae$ will remain small for this population since the normalization term is the total test set size. Moving forward, we will use the 4 larger subpopulations white-male, white-female, black or AA-male, and black or AA female to measure the results of our proposed algorithm. 
\begin{figure}
         \centering
         \includegraphics[width=\textwidth]{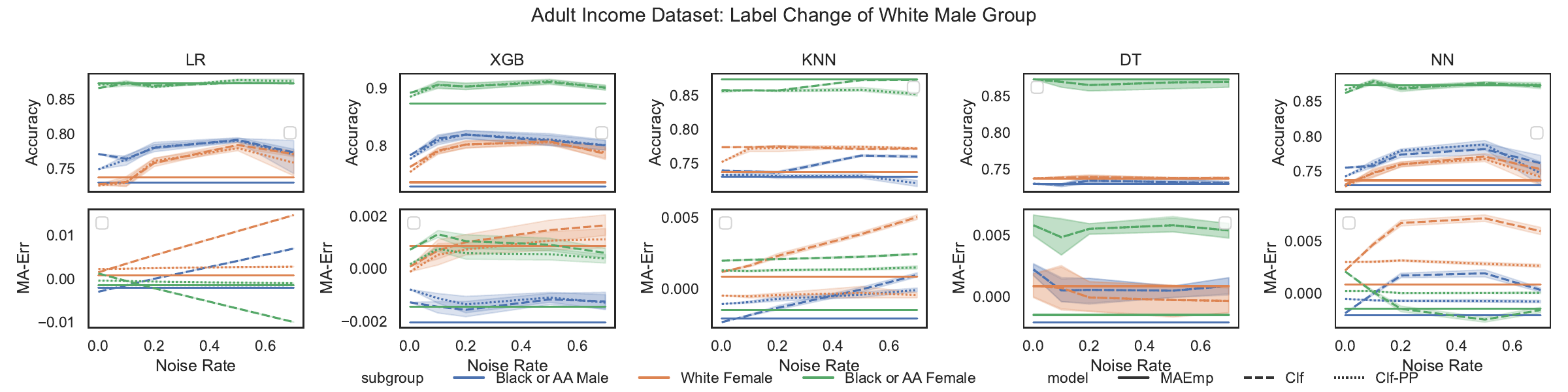}
         \caption{Comparison of different predictors on corrupted data with label change at various ratios on the ACSIncome Dataset. When labels are shifted from 0 to 1 in the white male group, baseline classifiers such as Logistic Regression, k-Nearest Neighbors, and Neural Network (MLP) also become more biased for other subgroups (i.e., increasing $\mae$). However, both the uniformly initialized and post processed predictor using Algorithm \ref{alg:ma-empirical} are robust to label change in other groups.}
         \label{fig:results-shift-white-male-income}
\end{figure}
\begin{figure}
         \centering
         \includegraphics[width=\textwidth]{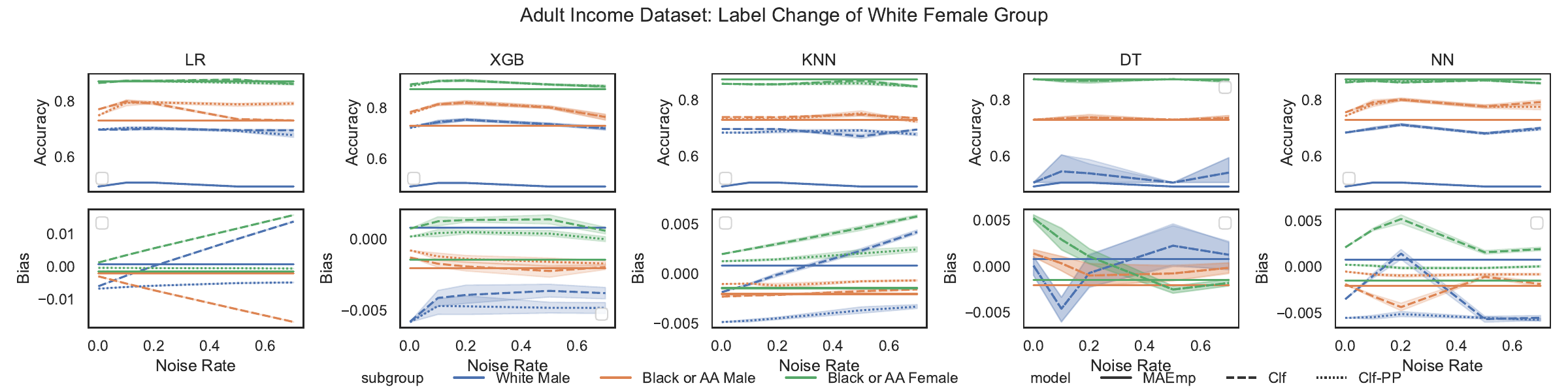}
         \caption{Comparison of different predictors on corrupted data with label change at various ratios on the ACSIncome Dataset: 0 to 1 label shift in the white female group. A similar result as modifying the white male group follows.}
         \label{fig:results-shift-white-female-income}
\end{figure}

\paragraph{Robustness to label change through Multi-Accuracy Boosting and Post Processing}
For the ACSIncome dataset, Figure \ref{fig:results-shift-white-male-income} (Also Figure 1 in the main text) shows the effect of randomly shifting labels in the White Male group on three other groups: White Female, Black or African American Male, and Black or African American Female. For Logistic Regression, we see that the base classifier (dotted line), results in increased bias as the noise rate increases while the \textsc{MAEmp} predictor retains a low and consistent level of multi-accuracy guarantee. Furthermore, post-processing the Linear Regression classifier with Empirical MABoost also results in similar accuracy to the original model and multi-accuracy regardless of the level of noise. A similar phenomenon is observed in the neural network (NN) and k-nearest neighbors models where \textsc{MAEmp} and \textsc{Clf-PP} are robust to label noise. Moreover, we observe that reducing $\mae$ comes at no cost to the overall accuracy: the accuracies of the post-processed predictors are comparable to the original predictors. In Figure \ref{fig:results-shift-white-female-income}, we observe a similar phenomenon but in the white male group as well as both black male and female groups.

For the ACS employment dataset, there is a significantly larger positive class ratio (Table \ref{tab:datasets}) but similar subgroup ratios since this dataset comes from the same state. In Figure \ref{fig:results-shift-white-male-employment}, we see a similar pattern to the income data set in terms of the multi-accuracy boosting algorithm producing predictors that are robust to changes in disjoint groups. Here we see that the accuracy of the post proceed predictor can be higher than before pre-processing (e.g., Logistic Regression - black female and black male subgroups). We see similar patterns in the predictor robustness when the white female group is modified in the ACSEmployment dataset. 

\begin{figure}
     \centering
     \begin{subfigure}
         \centering
         \includegraphics[width=\textwidth]{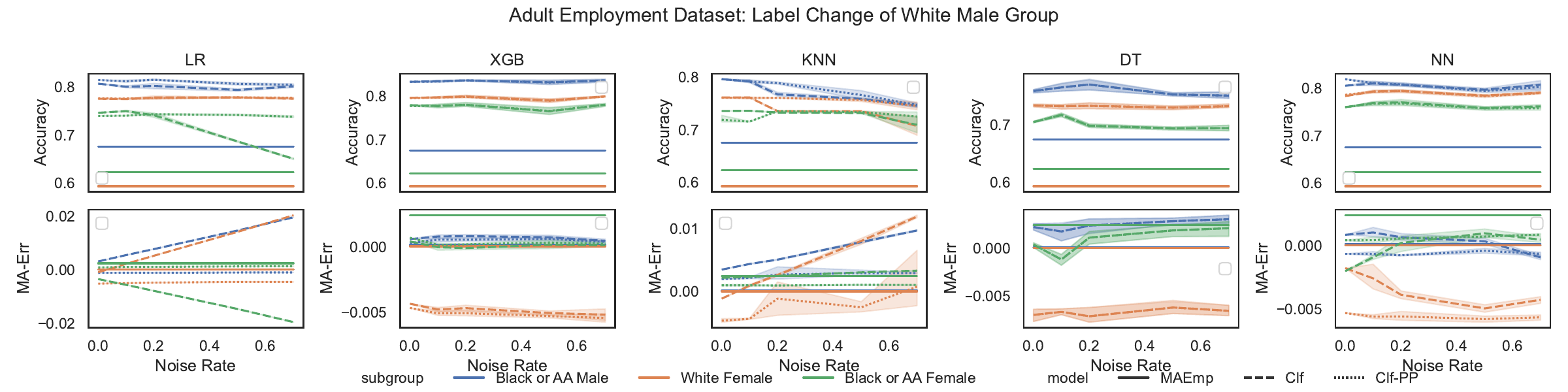}
         \caption{Comparison of different predictors on corrupted data in the white male group with label change at various ratios on the ACS Employment.}
         \label{fig:results-shift-white-male-employment}
     \end{subfigure}
     \hfill
     \begin{subfigure}
         \centering
         \includegraphics[width=\textwidth]{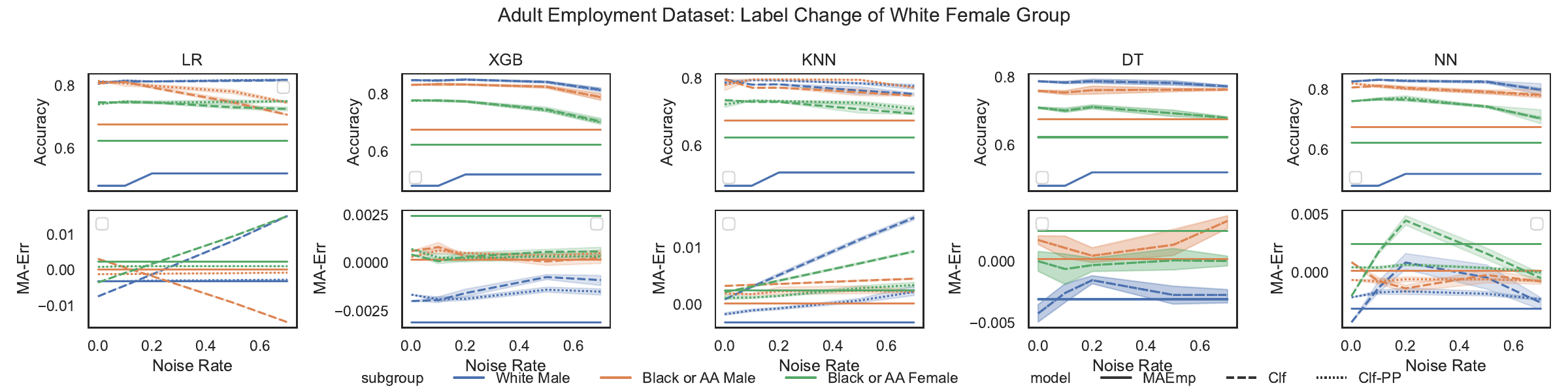}
         \caption{Comparison of different predictors on corrupted data in the white female group with label change at various ratios on the ACS Employment dataset.}
         \label{fig:results-shift-white-female-employment}
     \end{subfigure}
\end{figure}

For the ACS Coverage dataset, the predictors are trained to predict the probability that an individual is covered by public health insurance. Figure \ref{fig:results-shift-white-male-coverage} and \ref{fig:results-shift-white-female-coverage} illustrate the results for label shift from 0 to 1 in the white male and white female subgroups respectively. For shifts in both these subgroups, the impact on other groups is only significant for the logistic regression classifier. For logistic regression, we observe the base models suffering from more biased predictions ($\mae$) on unmodified groups as the level of noise increases while the post-processed logistic regression predictor retains a consistently low $\mae$ for all unmodified subgroups. 

\begin{figure}
     \centering
     \begin{subfigure}
         \centering
         \includegraphics[width=\textwidth]{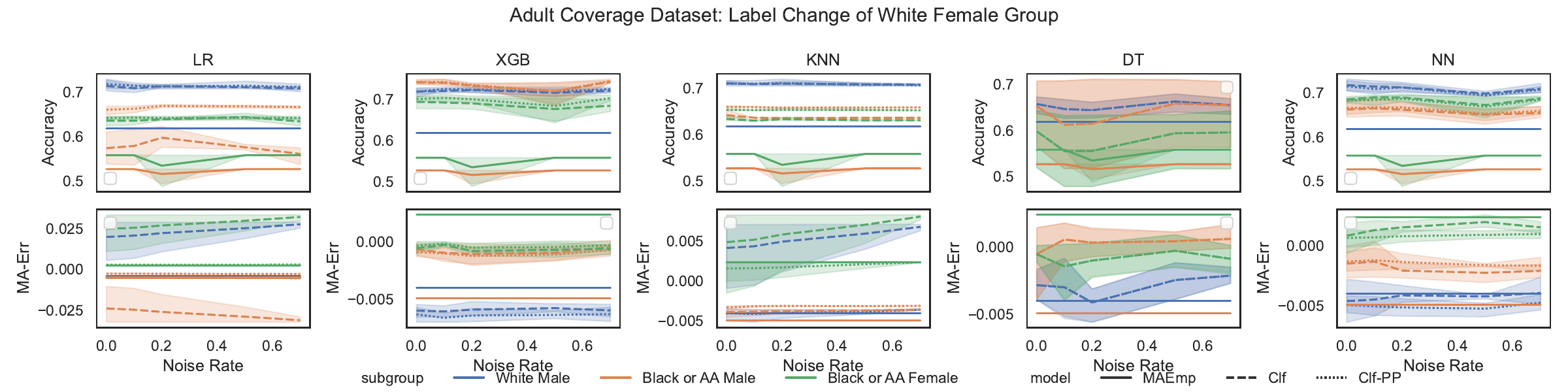}
         \caption{Comparison of different predictors on corrupted data in the white male group with label change at various ratios on the ACS Coverage dataset.}
         \label{fig:results-shift-white-male-coverage}
     \end{subfigure}
     \hfill
     \begin{subfigure}
         \centering
         \includegraphics[width=\textwidth]{coverage_shift_white-female.pdf}
         \caption{Comparison of different predictors on corrupted data in the white female group with label change at various ratios on the ACS Coverage dataset.}
         \label{fig:results-shift-white-female-coverage}
     \end{subfigure}
\end{figure}

\paragraph{Boosting Classifiers and Multi-Accuracy Error} Across all datasets (e.g., Income, Coverage, and Employment), we observe that the boosting classifiers are relatively calibrated already. For Decision Trees, in particular, the base and post-processed predictors are the same. This is due to the fact that we did not limit the depth of the model, the tree branches on subgroups since subgroup labels are a part of the features. A similar phenomenon exists for XGBoost where there is some randomization that allows some post-processing improvements but the base classifier is already multi-group robust. This is not surprising since for both Decision Trees and XGBoost, the ability to branch exactly based on group membership would reduce the effect of noise and corruption of unrelated subgroups. 

\subsection{Addition/Deletion}

We also demonstrate the robustness of Algorithm \ref{alg:ma-empirical} through corruptions to data in the form of additional data points. We employ a similar strategy as prior work designing subpopulation attacks~\cite{jagielski2021subpopulation} with the additional constraint that only data points outside of the target group can be added to create the poisoning dataset. To find data points to add that would affect the target group, we first cluster data points in a held-out set (not used for training or testing) using K-means. For each cluster where the target subgroup appears, we shift the labels of the group we are allowed to modify. We then add these shifted points from the held-out group to the training data. The amount of noise in this attack is scaled by how many times the identified data points are replicated before being added to the poisoned dataset. Algorithm \ref{alg:data-addition} describes the algorithm we use to generate a corrupted dataset $D'$ of size $m$. For these data poisoning attacks, we also target $0$ labels in each cluster to shift. This attack is more precise than random shifts in subgroups since there is a specific targeted subgroup. 

\begin{algorithm}
    \caption{Data Addition Algorithm}
    \label{alg:data-addition}
    \begin{algorithmic}
       \STATE {\bfseries Parameters:} modify group and target group: $C_{mod}, C_{tgt} \subseteq \{0, 1\}^X$, noise factor $\alpha \in [1, 10]$, num clusters $k$, cluster threshold $\gamma$, target: $t \in \{0, 1, *\}$
       \STATE {\bfseries Input: } clean dataset  $D = (x_1, y_1), ..., (x_n, y_n) \in X \times \{0, 1\}$, held out dataset $D_{aux}$
       \STATE {\bfseries Output: } corrupted dataset  $D' = (x_1, y_1'), ..., (x_m, y_m') \in X \times \{0, 1\}$
        \STATE $D' = \{\}$ 
        \STATE centers = \textsc{Cluster}$(D_{aux}, k)$
        \FORALL{$c \in $ centers}
        \STATE $D_c = \{(x, y) | \textsc{ClosestCenter}((x,y)) = c, (x, y) \in D_{aux} \}$
        \STATE $D_{tgt} = \{(x, y) | (x, y) \in C_{tgt}, (x, y) \in D_c \}$
        \STATE $D_{mod} = \{(x, y) | (x, y) \in C_{mod}, (x, y) \in D_c \}$
        \IF{$|D_{tgt}| \ge \gamma$}
        \FORALL{$(x, y) \in D_{mod}$}
        \IF{$y == t$} 
        \STATE $y' = 1 - y$
        \STATE $D' = D' \cup {(x, y')}^{\alpha}$ 
        \ENDIF
        \ENDFOR 
        \ENDIF
        \ENDFOR
        \STATE {\bfseries Return: }$D' = D' \cup D $
    \end{algorithmic}
\end{algorithm}
\begin{figure}
     \centering
     \begin{subfigure}
         \centering
         \includegraphics[width=\textwidth]{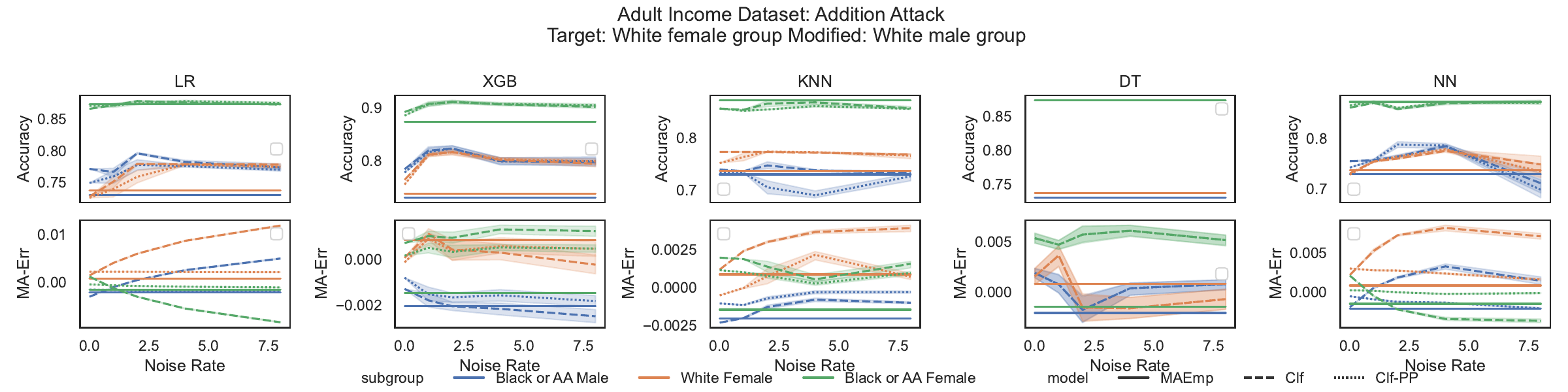}
         \caption{Comparison of different predictors on a corrupted dataset with injected data (Algorithm \ref{alg:data-addition}) at various levels on the ACSIncome Dataset. Baseline classifiers such as Logistic Regression, k-nearest Neighbors, and Neural Networks (MLP) become more biased (i.e., increasing $\mae$) for the target subgroup ($C_{tgt}$: White female) even though only data from the white male group ($C_{mod}$) has been added. However, both the uniformly initialized and post-processed predictor using Algorithm \ref{alg:ma-empirical} are robust to label change in other groups while preserving comparable accuracy.}
         \label{fig:results-addition-twhite-female-mwhite-male-income}
     \end{subfigure}
     \hfill
     \begin{subfigure}
         \centering
         \includegraphics[width=\textwidth]{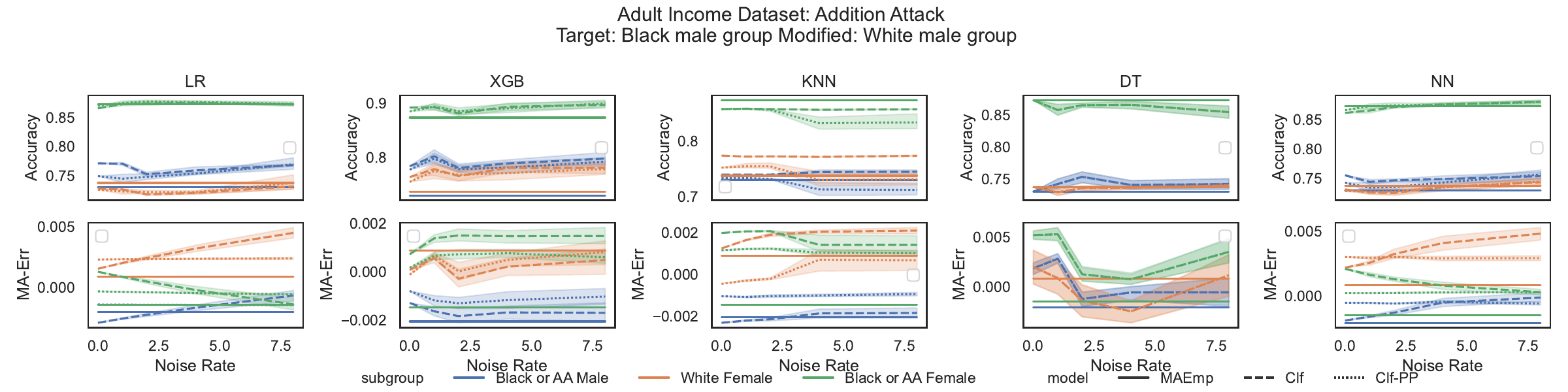}
         \caption{Comparison of different predictors on a corrupted dataset with injected data from the white male group ($C_{mod}$) with the black male group as the target group ($C_{tgt}$) at various levels on the ACSIncome Dataset.}
         \label{fig:results-addition-tblack-male-mwhite-male-income}
     \end{subfigure}
\end{figure}

In Figure \ref{fig:results-addition-twhite-female-mwhite-male-income} (Figure \ref{fig:type2-addition} in the main text), we see that even at low levels of noise (i.e., poisoned points are replicated once or twice), the Logistic Regression, k-Nearest Neighbors, and Neural Network base classifiers (\textsc{Clf}) exhibit worsened multi-accuracy error. However, with post-processing (\textsc{Clf-PP}), the mult-accuracy error remains consistent around 0 regardless of the noise ratio. Moreover, the benefits of low multi-accuracy error come without cost to the accuracy. These effects are not just on the target group, white female, but also appear in other subgroups such as the black male and female subgroups. Figure \ref{fig:results-addition-tblack-male-mwhite-male-income} shows the effects when the target group is the black male group where a similar effect exists.  

\begin{figure}
     \centering
     \begin{subfigure}
         \centering
         \includegraphics[width=\textwidth]{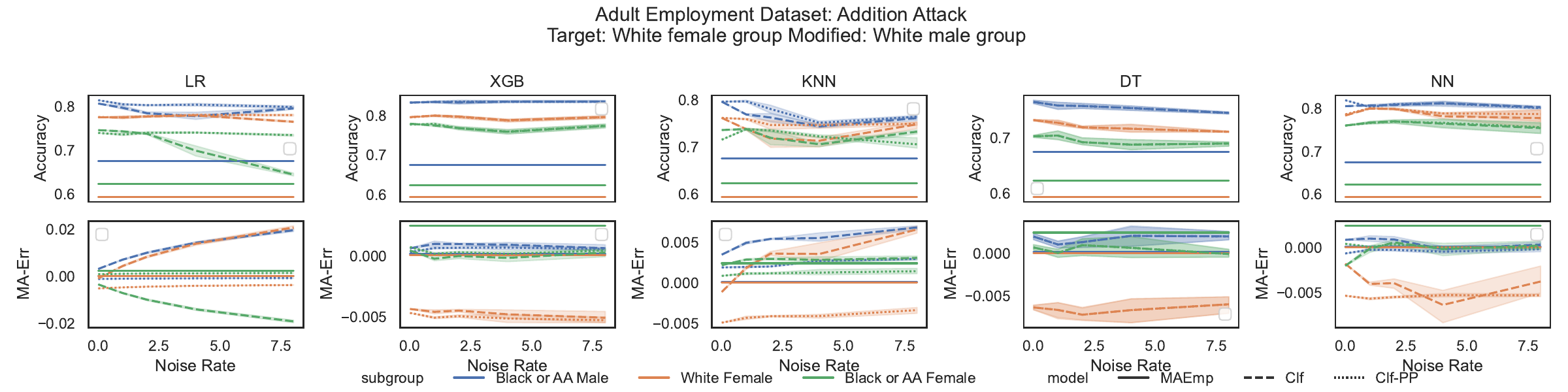}
         \caption{Comparison of different predictors on a corrupted dataset with injected data from the white male group ($C_{mod}$) with the white female group as the target group ($C_{tgt}$) at various levels on the ACSEmployment Dataset.}
         \label{fig:results-addition-twhite-female-mwhite-male-emp}
     \end{subfigure}
     \hfill
     \begin{subfigure}
         \centering
         \includegraphics[width=\textwidth]{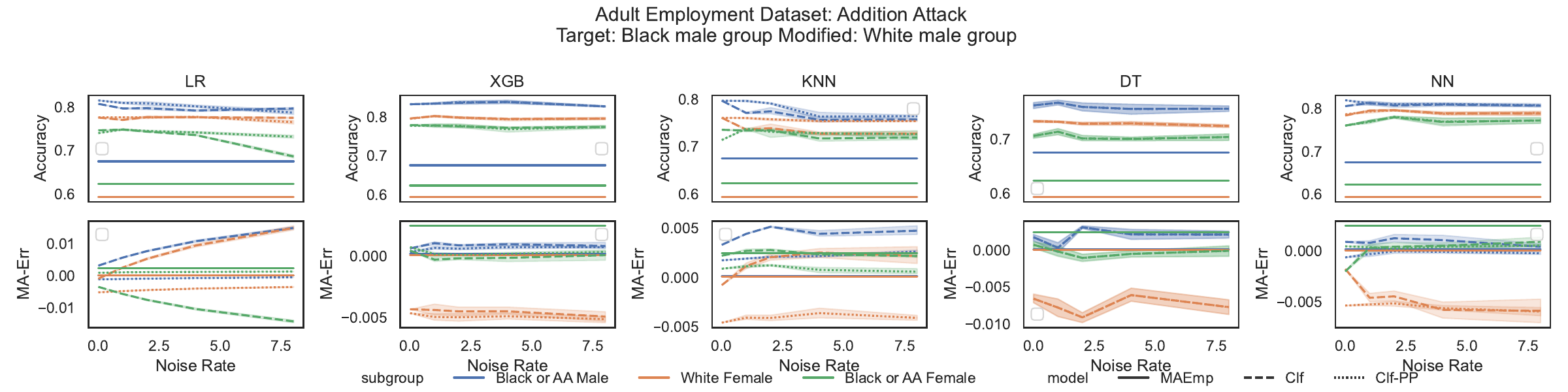}
         \caption{Comparison of different predictors on a corrupted dataset with injected data from the white male group ($C_{mod}$) with the black male group as the target group ($C_{tgt}$) at various levels on the ACSEmployment Dataset.}
         \label{fig:results-addition-tblack-male-mwhite-male-emp}
     \end{subfigure}
\end{figure}

\begin{figure}
     \centering
     \begin{subfigure}
         \centering
         \includegraphics[width=\textwidth]{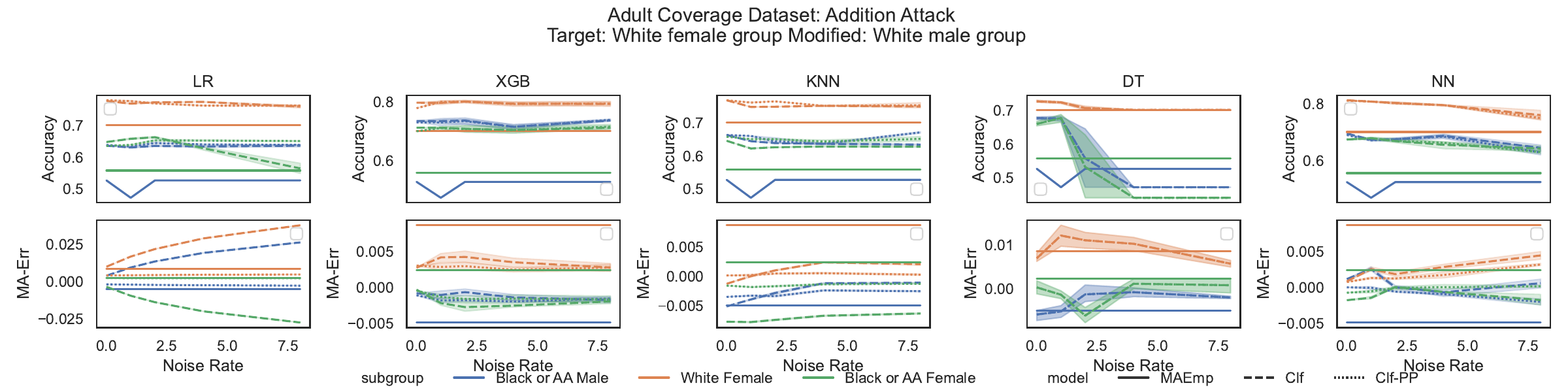}
        \caption{Comparison of different predictors on a corrupted dataset with injected data from the white male group ($C_{mod}$) with the white female group as the target group ($C_{tgt}$)  at various levels on the ACS Coverage Dataset.}
         \label{fig:results-addition-twhite-female-mwhite-male-cov}
     \end{subfigure}
     \hfill
     \begin{subfigure}
         \centering
         \includegraphics[width=\textwidth]{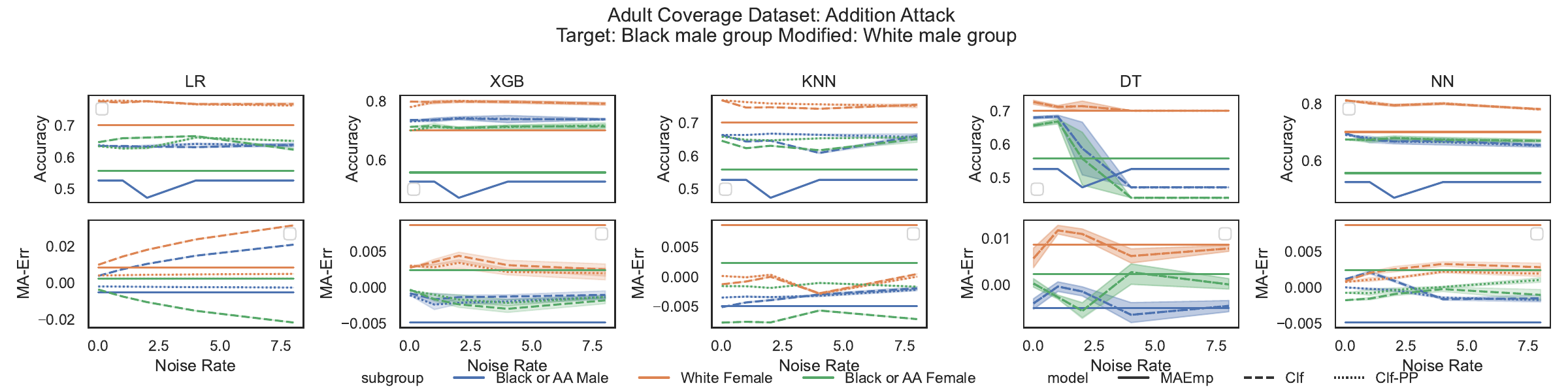}
         \caption{Comparison of different predictors on a corrupted dataset with injected data from the white male group ($C_{mod}$) with the black male group as the target group ($C_{tgt}$) at various levels on the ACS Coverage Dataset.}
         \label{fig:results-addition-tblack-male-mwhite-male-cov}
     \end{subfigure}
\end{figure}

\end{document}